\pdfoutput=1

\documentclass[11pt]{article}

\usepackage[final]{acl}

\usepackage{times}
\usepackage{latexsym, amsthm, mathtools}
\usepackage{appendix, subfigure}
\usepackage{tcolorbox, multicol, multirow, makecell, amsfonts, booktabs, colortbl}
\usepackage{algorithm, algpseudocode}

\tcbuselibrary{skins, breakable, theorems}
\usepackage[T1]{fontenc}

\usepackage[utf8]{inputenc}

\usepackage{microtype}
\usepackage{float}
\usepackage{inconsolata}
\usepackage{url}

\usepackage{graphicx}

\title{\textit{Adapt to Thrive!} Adaptive Power-Mean Policy Optimization for Improved LLM Reasoning}

\author{
  Yiming Huang\textsuperscript{1}, 
  Zhenbo Shi\textsuperscript{1}, 
  Shuzheng Gao\textsuperscript{3}, 
  Cuiyun Gao\textsuperscript{1}\thanks{Corresponding authors.}, 
  Peiyi Han\textsuperscript{1,2}, 
  Chuanyi Liu\textsuperscript{1,2}\footnotemark[1]
  \\
  \textsuperscript{1}Harbin Institute of Technology, Shenzhen \\
  \textsuperscript{2}Peng Cheng Laboratory \\
  \textsuperscript{3}The Chinese University of Hong Kong \\
  \small \textbf{Correspondence:} 
  \href{mailto:24b951042@stu.hit.edu.cn,2023311604@stu.hit.edu.cn,szgao98@gmail.com}{
    \{24b951042,2023311604\}@stu.hit.edu.cn, szgao98@gmail.com,
  }
  \\ 
  \small \href{mailto:gaocuiyun@hit.edu.cn,hanpeiyi@hit.edu.cn,liuchuanyi@hit.edu.cn}{
    \{gaocuiyun,hanpeiyi,liuchuanyi\}@hit.edu.cn
  }
} 

\begin{document}
\maketitle
\begin{abstract}
Reinforcement Learning with Verifiable Rewards (RLVR) is an essential paradigm that enhances the reasoning capabilities of Large Language Models (LLMs). However, existing methods typically rely on static policy optimization schemes that misalign with the model's evolving reasoning capabilities. To address this issue, we propose \textbf{Adaptive Power-Mean Policy Optimization (APMPO)}, which comprises two main innovations: Power-Mean Policy Optimization (PMPO) and Feedback-Adaptive Clipping (FAC). Specifically, PMPO introduces a generalized power-mean objective. This enables the model to adaptively transition from the signal‑amplifying behavior of the arithmetic mean to the consistency‑enforcing behavior of the geometric mean. FAC adaptively adjusts clipping bounds based on real-time reward statistics to overcome the limitations of static mechanisms. Capitalizing on these innovations, APMPO improves learning dynamics and reasoning performance. Extensive experiments on nine datasets across three reasoning tasks showcase the superiority of APMPO over state-of-the-art RLVR-based baselines. For instance, APMPO boosts the average Pass@1 score on mathematical reasoning benchmarks by 3.0 points compared to GRPO when using Qwen2.5-3B-Instruct.

\end{abstract}

\section{Introduction}
\begin{quote}
\textit{The measure of intelligence is the ability to change.}
\end{quote}
\begin{flushright}
\textit{- Albert Einstein}
\end{flushright}

Enhancing the reasoning capabilities of Large Language Models (LLMs) \cite{hurst2024gpt, yang2024qwen2} has become a central research focus. As a promising approach, reinforcement learning (RL) empowers LLMs to surpass simple pattern matching by refining their decision‑making strategies based on task‑specific feedback \cite{guo2025deepseek}. Within this field, Reinforcement Learning with Verifiable Rewards (RLVR) \cite{wu2025invisible, wen2025reinforcement} is recognized for effectively enhancing complex reasoning. Leveraging outcome-based rewards, RLVR has driven notable progress in multiple domains such as mathematics \cite{chen2025bridging}, coding \cite{ye2025process}, and multi-modal reasoning \cite{wang2025geometryzero}.

A cornerstone algorithm in RLVR is Group Relative Policy Optimization (GRPO) \cite{shao2024deepseekmath}, which has gained widespread attention. By eliminating the need for a value model, GRPO reduces memory consumption and streamlines the training process \cite{ramesh2024group}. These benefits have led to several GRPO variants, each designed to address specific limitations in policy optimization \cite{yu2025dapo, zhao2025geometric}. Despite these efforts, existing methods face inherent limitations that impede their wider adoption:

\textbf{Limitation 1: Static objective functions misalign with the model's evolving reasoning performance.} Current RLVR-based methods typically rely on fixed objective functions throughout training, ignoring the need to adapt sensitivity to reward signals at different learning phases.  
For instance, GRPO \cite{shao2024deepseekmath} and DAPO \cite{yu2025dapo} use the arithmetic mean, which is
highly sensitive to high-reward outliers. While this sensitivity can amplify rare high‑reward signals, it often drives the policy to overfit specific solutions and induces entropy collapse. In contrast, GMPO \cite{zhao2025geometric} adopts the geometric mean to aggregate rewards from multiple reasoning paths. Since a single low reward can sharply reduce the geometric mean, GMPO tends to discard promising but unstable reasoning paths. This makes GMPO less effective during early training when correct reasoning paths are scarce. Consequently, these methods cannot adjust their 
sensitivity to match evolving learning dynamics. These observations highlight the need for an \textit{adaptive objective function} that balances the amplification of high‑value signals with sustained policy optimization for improved reasoning.

\textbf{Limitation 2: Static policy optimization constraints overlook variations in reward signal stability.} Standard algorithms such as GRPO employ clipping mechanisms to stabilize training. They typically enforce fixed constraint thresholds irrespective of the statistical stability of reward signals. This static design is suboptimal, since the stability of reward signals varies across training batches. When rewards within a batch are statistically stable, they provide a clearer direction for policy improvement. This allows for more aggressive updates within a wider trust region. Conversely, batches with highly fluctuating rewards reflect ambiguity in the policy’s decisions and require tighter trust regions to mitigate unstable policy updates. Ignoring the stability of reward signals can lead to misguided policy updates, ultimately degrading model performance \cite{huang2025low, yu2025dapo, yoonconfpo2025}. This motivates the design of an \textit{adaptive clipping mechanism} that adjusts clipping bounds based on the real-time statistical stability of reward signals.

In light of the above limitations, this work investigates the following key research question:

\textbf{\textit{How to design an adaptive algorithm that aligns learning objectives and policy optimization constraints with the model's learning process?}}

To answer this question, we propose \textbf{Adaptive Power-Mean Policy Optimization (APMPO)}, a novel RLVR-based algorithm designed to strengthen the reasoning capabilities of LLMs. Specifically, APMPO integrates two innovations: (1) \textbf{Power-Mean Policy Optimization (PMPO)}: To address \textit{Limitation 1}, PMPO introduces a power-mean formulation that adaptively modulates the objective behavior between the arithmetic and geometric mean objectives. By adaptively balancing these two extremes, PMPO can mitigate training instability and promote the discovery of high-value signals. (2) \textbf{Feedback-Adaptive Clipping (FAC)}: To address \textit{Limitation 2}, FAC introduces a feedback-driven mechanism that modulates clipping bounds based on statistical stability of real-time reward signals. In this design, reward signals serve as feedback proxies for evaluating policy reliability. Consistent reward feedback requires larger policy updates, whereas unstable reward feedback demands stricter policy constraints. Collectively, these innovations enable APMPO to achieve adaptive policy optimization. Extensive experiments on nine benchmarks spanning three reasoning tasks demonstrate the superiority of APMPO over state-of-the-art RLVR-based baselines.

Before delving into the details, the main contributions of this work are summarized as follows:

(1) We identify and formalize two critical limitations in current RLVR-based methods, particularly the static nature of their objective functions and clipping mechanisms.     These limitations hinder the adaptive tuning of policy optimization strategies in response to changes in model capability.  

(2) We introduce APMPO, a novel RLVR-based algorithm that integrates a power mean-based objective function with a feedback-adaptive clipping mechanism. This design enables continuous adjustment of policy optimization strategies based on real‑time training dynamics.  

(3) Extensive experimental results on nine datasets across three tasks substantiate the superiority of APMPO in enhancing the reasoning accuracy of LLMs.

\begin{figure*}[t]
\centering
\subfigure[Training dynamics of reward curves using different RL-based methods.]{
		\includegraphics[scale=0.23]{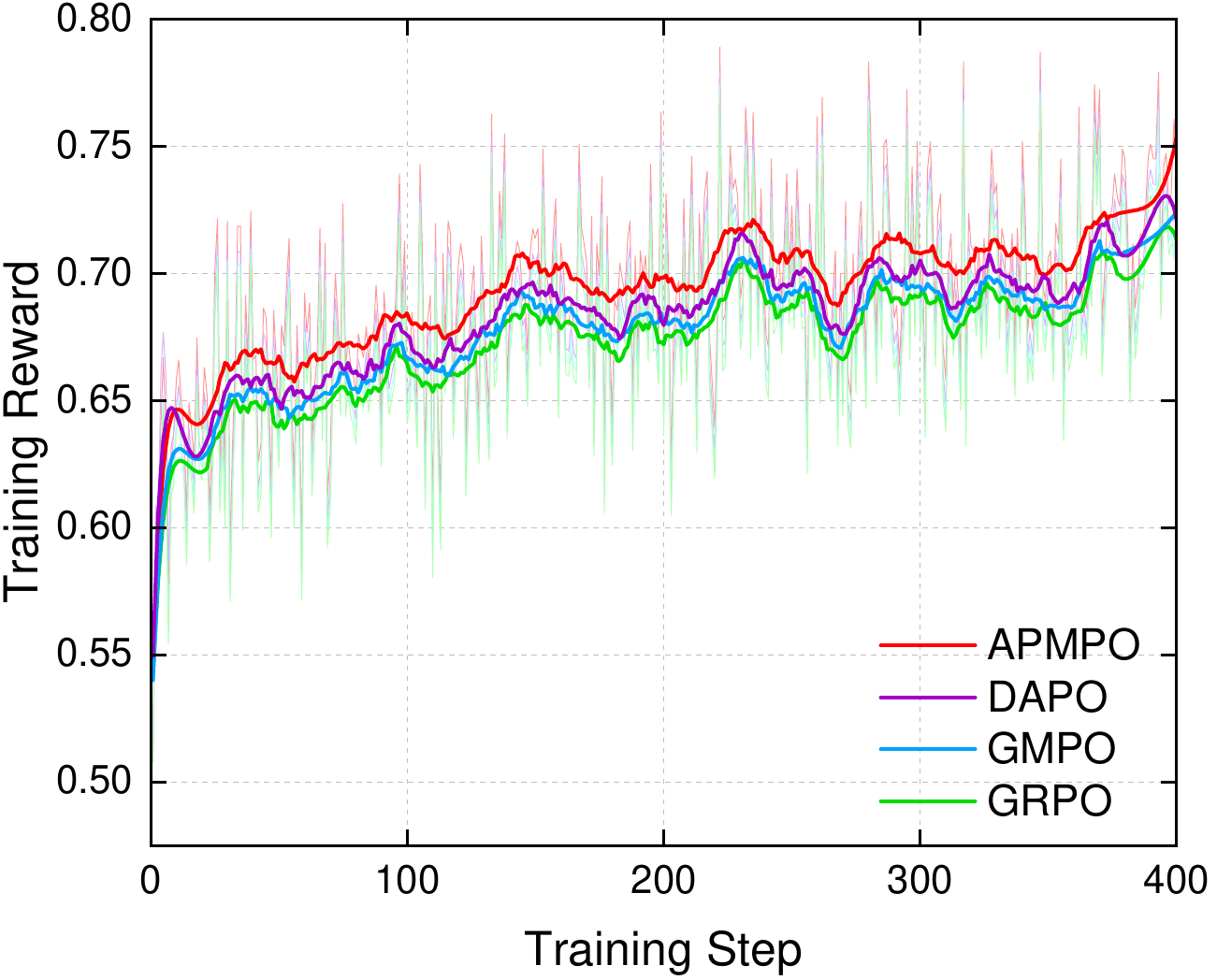}}
\subfigure[Training dynamics of entropy curves using different RL-based methods.]{
		\includegraphics[scale=0.23]{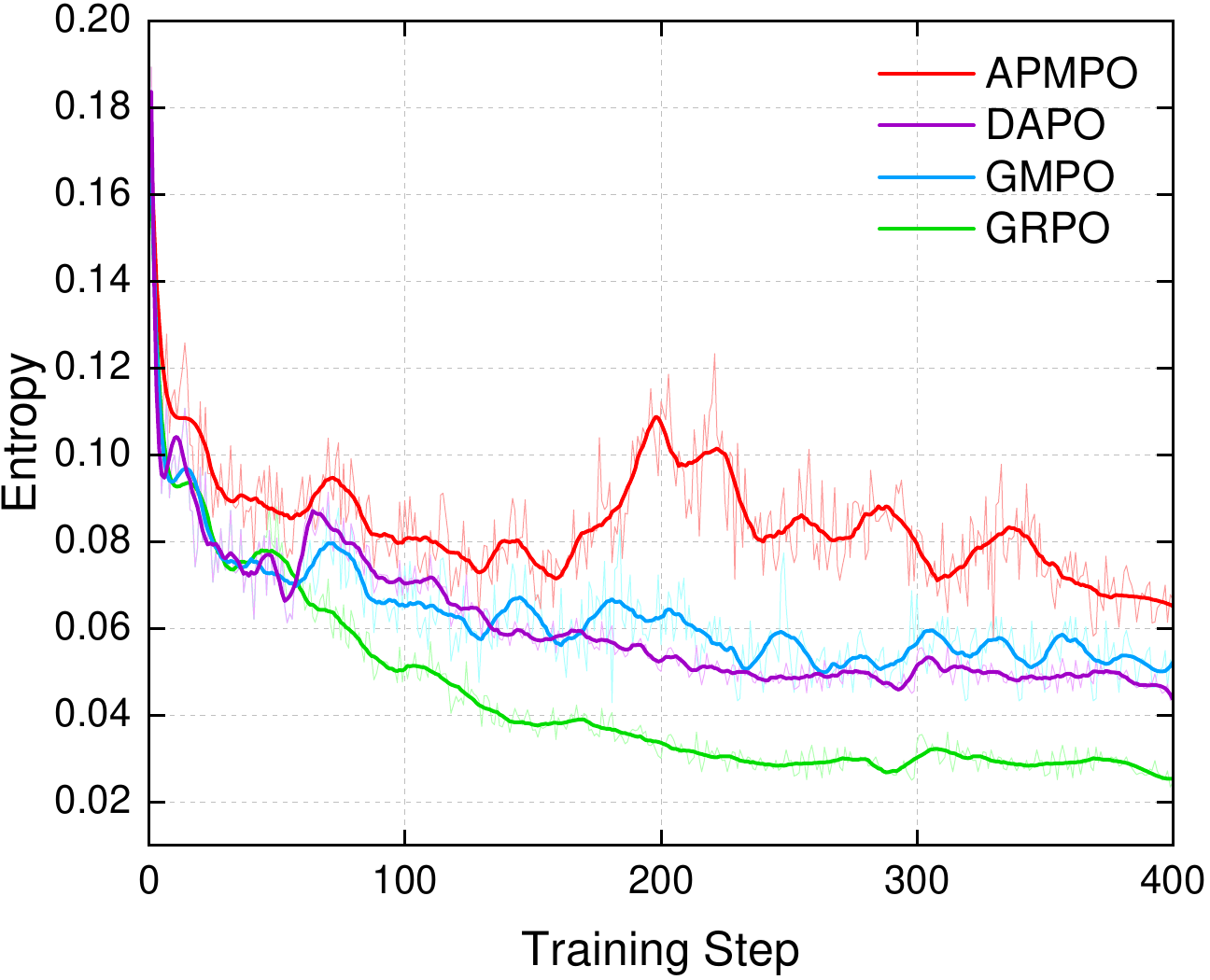}}
\subfigure[Wall-clock training time using different RL-based methods.]{
		\includegraphics[scale=0.23]{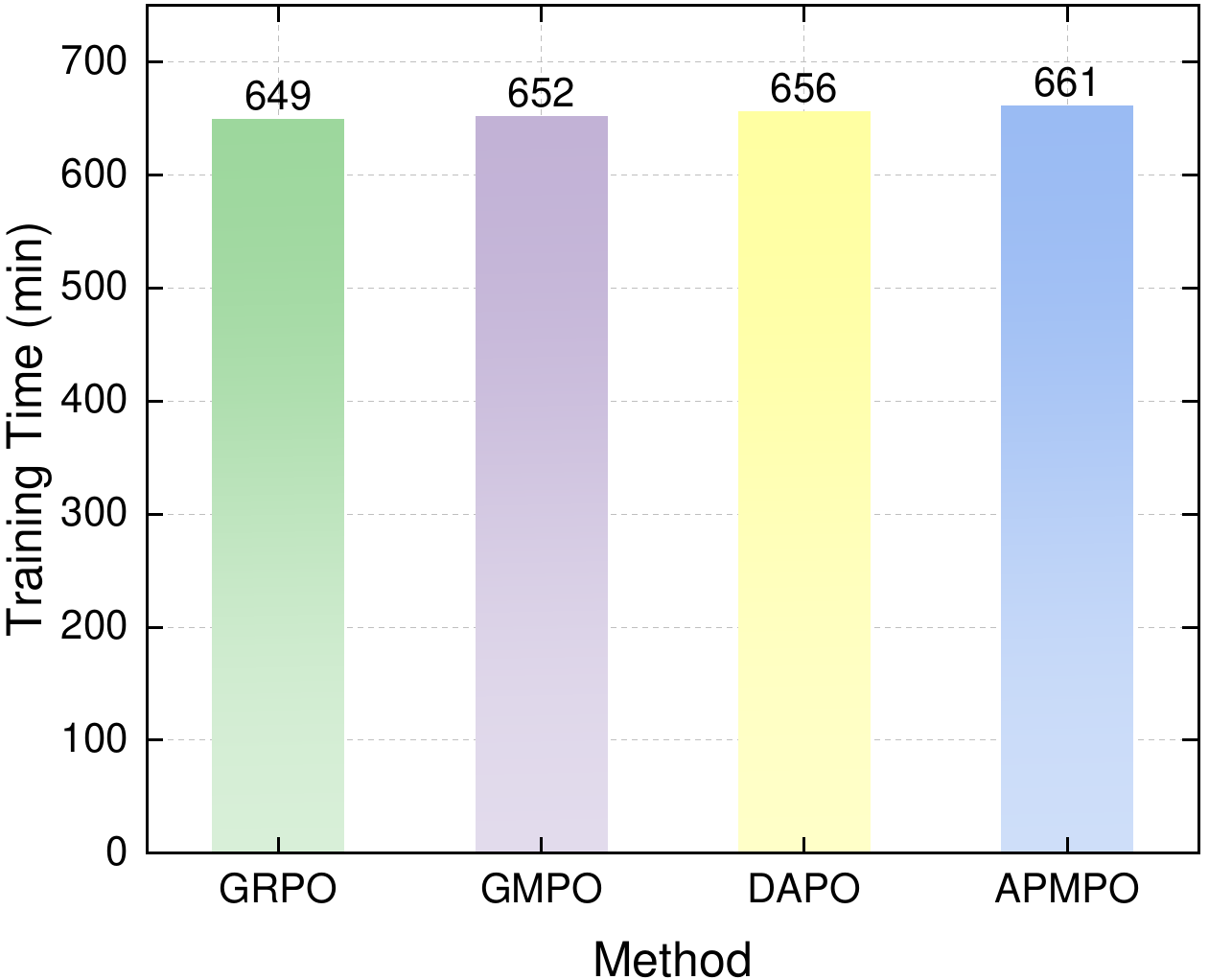}}
\caption{Illustrations of training dynamics in terms of training rewards, policy entropy, and training time using the MATH training dataset.}
\label{fig:training}
\end{figure*}

\section{Related Work}

GRPO \cite{shao2024deepseekmath} has become prominent in RLVR. Its core innovation lies in directly estimating advantages from reward scores across multiple sampled solutions, thereby eliminating the need for an explicit value model. To promote stable training, GRPO constrained policy updates via a clipping function controlled by a fixed hyperparameter. This approach has inspired several variants aimed at enhancing the reasoning performance of LLMs \cite{yu2025dapo, liu2025understanding, chu2025gpg, zhao2025geometric}. For instance, DAPO \cite{yu2025dapo} introduced dynamic sampling to explicitly filter out zero-advantage samples, thereby improving gradient quality at the cost of reduced data efficiency. GMPO \cite{zhao2025geometric} tackled training instability by optimizing the geometric mean of rewards instead of the arithmetic mean, making it inherently less sensitive to reward outliers. Additional details on these algorithms are provided in Appendix~\ref{appendix:compared}.

Despite these advances, existing approaches remain constrained by static design choices. Specifically, objective functions and trust-region bounds remain static across all training batches. This static design limits their adaptability to varying reward distributions and training dynamics. As a result, the ability to adaptively adjust policy optimization behavior in response to real‑time learning signals is critical for sustained performance improvement. In light of this, we propose APMPO to promote superior policy optimization.

\section{Preliminary Analysis}
The motivation for APMPO arises from a preliminary analysis of the training dynamics in GRPO and GMPO, centering on training reward and policy entropy. 

\subsection{Motivation for PMPO: The Dilemma of Static Objective Functions}
Figures~\ref{fig:training}(a) and (b) reveal the inherent limitations of static objective functions.
Specifically, the arithmetic-mean objective in GRPO exhibits \textit{excessive sensitivity to high-reward outliers}. While this sensitivity facilitates early identification of promising signals, it often drives the policy toward early convergence on suboptimal strategies. This is evidenced by a sharp entropy collapse (green curve). 

In contrast, GMPO applies a geometric-mean objective to enforce consistency. This approach alleviates entropy collapse and path-specific overfitting, leading to superior policy updates than GRPO (blue curve). Nonetheless, this strict consensus requirement reduces sensitivity. Mathematically, the geometric mean functions as a global filter that weakens the influence of distinct high‑value signals. As a result, GMPO risks impeding the acquisition of correct reasoning in early training phases. 

These findings emphasize the need for an adaptive objective function that \textit{balances signal amplification with distributional consistency}. This insight directly motivates PMPO, which adaptively interpolates between these distinct learning patterns. A detailed analysis of GRPO and GMPO is presented in Appendix~\ref{sec:appendix_example}.

\subsection{Motivation for Adaptive Clipping: The Need for Adaptive Trust Regions}
Current RL‑based methods employ fixed clipping hyperparameters to impose static constraints on policy updates. However, this design is suboptimal given the continual changes in reward distributions. Notably, the raw reward signals in Figure~\ref{fig:training}(a) (thin lines) fluctuate significantly, reflecting disparate levels in statistical stability across batches. Current methods inherently assume stationary reward statistics, an assumption that rarely holds in practice. Consequently, static clipping bounds fail to adapt to the reward distributional shifts. They can be overly restrictive for batches exhibiting high reward stability and excessively permissive for those with unstable reward patterns. Therefore, the policy struggles to converge during stable phases and remains exposed to detrimental updates during unstable phases. These observations motivate the design of FAC, which adaptively modulates clipping bounds based on the real-time statistical properties of the reward signal.

\begin{figure*}[h]
\centering 
\includegraphics[width=0.97\textwidth]{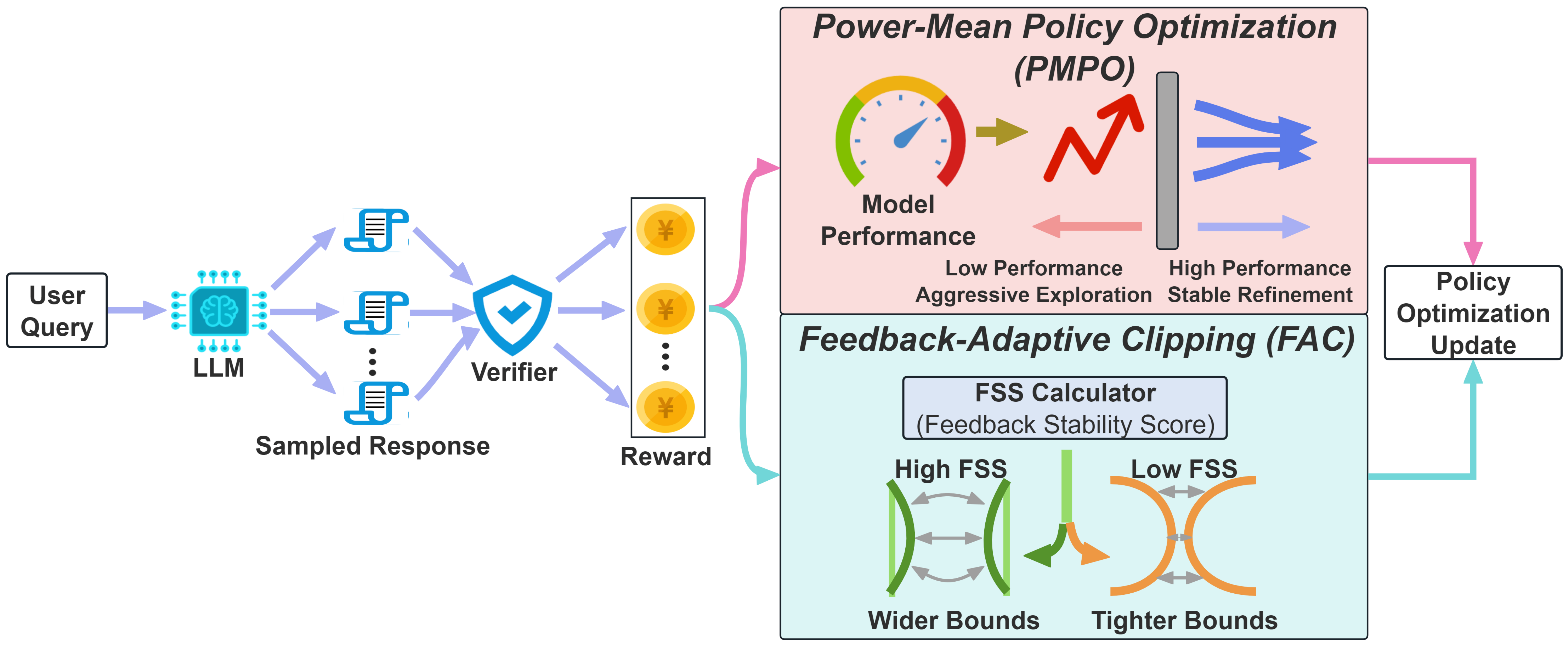}
\caption{Illustration of APMPO, which consists of Power-Mean Policy Optimization (PMPO) and Feedback-Adaptive Clipping (FAC).}
\label{fig:framework}
\end{figure*}

\section{Methodology}\label{sec:method}

As depicted in Figure \ref{fig:framework}, APMPO 
is proposed to enhance reasoning capabilities of LLMs. It consists of Power-Mean Policy Optimization (PMPO) and Feedback-Adaptive Clipping (FAC). In the following, some notations and preliminaries are briefly introduced before each of its key innovations is elaborated in detail.

\subsection{Notations and Preliminaries}
\noindent \textbf{Group Sampling.} For each input query $q$, a group of $G$ responses is sampled from the old policy $\pi_{\theta_{\text{old}}}$:
\begin{equation}
\{o_i\}_{i=1}^G \sim \pi_{\theta_{\text{old}}}(\cdot | q)
\label{eq:group_sampling}
\end{equation}
where each $o_i$ comprises a sequence of tokens $\{o_{i,1},o_{i,2}, \dots\}$, and a scalar reward $R_i$ is assigned to each completed response.

\noindent \textbf{Group-Normalized Advantage.} The advantage $\hat{A}_{i}$ is computed for each sample given as:
\begin{equation}
\hat{A}_{i} = \frac{R_i - \mu_R}{\sigma_R + \delta}
\label{eq:group_advantage}
\end{equation}
where $\mu_R$ and $\sigma_R$ denote the mean and standard deviation of rewards within the group. 

\noindent \textbf{Importance Sampling Ratio.} For each token $o_{i,t}$, the importance sampling ratio $r_{i,t}(\theta)$ measures the change in action probability between the current policy $\pi_\theta$ and the old policy $\pi_{\theta_\text{old}}$:
\begin{equation}
    r_{i,t}(\theta) = \frac{\pi_\theta(o_{i,t} | q, o_{i,<t})}{\pi_{\theta_{\text{old}}}(o_{i,t} | q, o_{i,<t})}
\label{eq:importance_ratio}
\end{equation}


\subsection{Power-Mean Policy Optimization (PMPO)} \label{PMPO}

Existing RL-based methods typically employ static objective functions that exhibit contrasting limitations. Concretely, the arithmetic‑mean objective in GRPO induces overly aggressive policy updates by prioritizing outliers, thereby causing mode collapse. Conversely, conservative formulations such as GMPO overly penalize reward variance and hinder the amplification of high-reward outliers.

To resolve this dilemma, we introduce \textbf{Power-Mean Policy Optimization (PMPO)} that adaptively interpolates between these two learning patterns. The key insight is that the degree of policy‑update aggressiveness can be explicitly controlled by the choice of the mean operator. This design builds on the theoretical insight that both the arithmetic and geometric means are special cases of the generalized power mean (see Appendix~\ref{appendix:theo1}). Accordingly, PMPO employs a generalized power mean where the exponent $p$ is adaptively determined by the model’s real‑time performance. Specifically, we use the average reward of each batch $\mu_R \in [0,1]$, which reflects the model's learning progress. To facilitate a smooth transition from signal discovery to stable refinement, we adopt an exponential decay function for the exponent $p$:
\begin{equation}
p = \exp(-\gamma \cdot \mu_R)
\label{eq:dynamic_p}
\end{equation}
where $\gamma$ regulates the sensitivity of $p$ to performance variations. This formulation naturally implements a performance-driven scheduling strategy:

(1) \textbf{Exploration Phase (Low $\mu_R$):} When performance is low, $p$ approaches 1. This encourages signal amplification by prioritizing high-reward outliers, thereby facilitating the rapid discovery of rare correct solutions.  

(2) \textbf{Consolidation Phase (High $\mu_R$):} As performance improves, $p$ decays toward 0. This transitions the policy into a consistency-enforcing phase, prioritizing reward stabilization to regularize the policy against overfitting.

This adaptive mechanism empowers PMPO to continuously adjust its learning dynamics, seamlessly interpolating between local sensitivity and global stability within a unified framework. The complete objective function formulation is detailed in Section~\ref{subsec:rl_training} (See Appendix \ref{appendix:theory_dynamics} for more analysis).

\subsection{Feedback-Adaptive Clipping (FAC)} \label{subsec:adaptive_clipping}
Current RL-based approaches typically utilize a fixed clipping ratio $\epsilon$ to constrain policy updates, thereby enforcing a uniform trust region across all batches. Nonetheless, this static design cannot adapt to the fluctuating statistical stability of training batches. Consequently, a fixed $\epsilon$ is overly restrictive for statistically stable batches while remaining overly permissive for noisy batches.

To address this limitation, we propose \textbf{Feedback-Adaptive Clipping (FAC)}, which adaptively adjusts clipping bounds based on the real-time reward distributions. FAC assesses the reward feedback to estimate its suitability for guiding policy updates. The key insight is that a high-stability feedback batch exhibits a strong average reward signal and high internal consistency. Accordingly, we formalize this notion by using the \textit{Feedback Stability Score (FSS)} as follows:
\begin{equation}
\text{FSS} = \frac{\mu_{R}}{\sigma_{R} + \delta}
\label{eq:snr}
\end{equation}
where $\mu_{R}$ and $\sigma_{R}$ are the mean and standard deviation of the rewards within a batch, respectively. $\delta$ is a small constant for numerical stability. Notably, FSS serves as a critical non-linear scaling factor that modulates the effective signal magnitude. Intuitively, a higher FSS signifies a high-confidence batch with stable feedback, thereby justifying an expanded trust region to accelerate learning. Conversely, a low FSS implies that the signal is dominated by high uncertainty, necessitating a conservative clipping bound to alleviate potential policy degradation. Additional analysis on FSS is shown in Appendix \ref{appendix:trust}.

Based on the computed FSS, the clipping bound is adaptively updated. Specifically, FSS is mapped to an adaptive upper clipping bound $\epsilon_{\text{ada}}$ given as:
\begin{equation}
\epsilon_{\text{ada}} = \epsilon_{\text{min}} + (\epsilon_{\text{max}} - \epsilon_{\text{min}}) \cdot \tanh(\text{FSS})
\label{eq:adaptive_epsilon}
\end{equation}
where $\epsilon_{\text{min}}$ and $\epsilon_{\text{max}}$ are the predefined minimum and maximum clipping bounds, respectively. The $\tanh(\cdot)$ function smoothly maps unbounded FSS values to the range $(0,1)$, enabling the clipping bound to adaptively widen in stable phases and narrow in unstable phases. A detailed analysis of FAC is provided in Appendix \ref{appendix:theo2}.

\subsection{RL Training} \label{subsec:rl_training}
PMPO and FAC are jointly integrated to construct the final objective of APMPO. The training process proceeds in three logically connected steps:

\textbf{Step 1:} An adaptive clipping function $\rho_{i,t}(\theta)$ is defined to establish an asymmetric trust region. While the lower bound is fixed to ensure baseline stability, the upper bound is adaptively modulated by $\epsilon_{\text{ada}}$ 
from Eq. (\ref{eq:adaptive_epsilon}) to adapt to signal quality:
\begin{equation}
\rho_{i,t}(\theta) = \text{clip}(r_{i,t}(\theta), 1 - \epsilon_{\text{low}}, 1 + \epsilon_{\text{ada}})
\label{eq:effective_ratio}
\end{equation}
where $\epsilon_{\text{low}}$ is a fixed lower bound to prevent excessive policy changes. This mechanism allows the policy to aggressively exploit high-quality signals while remaining conservative under noisy feedback. Further analysis on this asymmetric design is shown in Appendix \ref{appendix:clipping asymmetry}.

\textbf{Step 2:} To satisfy the non-negative constraint of the power-mean operator, the computation is decoupled into a non‑negative magnitude term and a directional term. The token-level magnitude $\phi_{i,t}(\theta)$ is defined as:
\begin{equation}
\phi_{i,t}(\theta) = \left| \min \left( r_{i,t}(\theta)\hat{A}_{i}, \rho_{i,t}(\theta)\hat{A}_{i} \right) \right|
\label{eq:token_magnitude}
\end{equation}

\textbf{Step 3:} Finally, APMPO aggregates the token-level magnitudes via the power‑mean operator and introduces a directional control term. The per-sequence objective is defined as:
\begin{equation}
\mathcal{J}_i(\theta) = \underbrace{\left[\frac{1}{|o_i|} \sum_{t=1}^{|o_i|} \left(\phi_{i,t}(\theta) \right)^p \right]^{1/p}}_{\text{PMPO}} \cdot \underbrace{\text{sgn}(\hat{A}_i)}_{\text{Directional Control}}\label{eq:per_sequence_objective}
\end{equation}
where $\text{sgn}(\hat{A}_i) \in \{-1,1\}$ ensures that positive advantages drive policy maximization while negative advantages incur penalties. The complete objective function $\mathcal{J}(\theta)$ is then given as:
\begin{equation}
\mathcal{J}(\theta) = \frac{1}{G} \sum_{i=1}^G \mathcal{J}_i(\theta)- \beta D_{\text{KL}}(\pi_\theta || \pi_{\theta_{\text{ref}}})
\label{eq:final_objective}
\end{equation}
where $\pi_{\theta_{\text{ref}}}$ is the reference policy. A detailed gradient derivation, convergence analysis, and pseudocode are provided in Appendices \ref{appendix:gradient}, \ref{appendix:convergence}, and \ref{appendix:pseudo}, respectively.

\begin{table*}[t]
\centering
\small
\setlength{\tabcolsep}{0.8mm}
\begin{tabular}{cccccccc}
\toprule
\textbf{Method} & \textbf{Math500} & \textbf{AIME24} & \textbf{AIME25} & \textbf{AMC23} & \textbf{Minerva} & \textbf{Olympiad}& \textbf{Avg.}  \\
\midrule
\multicolumn{8}{c}{\textit{Qwen2.5-Math-1.5B-Instruct}} \\
\midrule
Base & 74.2 & 10.0 / 16.7 & 3.3 / 10.0 & 47.5 / 70.0 & 28.7 & 35.2 & 33.2 / 32.2  \\
GRPO & 75.2$_{\pm 0.4}$ & 13.3$_{\pm 0.0}$ / 16.7$_{\pm 0.0}$ & 13.3$_{\pm 0.0}$ / 16.7$_{\pm 0.0}$ & 52.5$_{\pm 1.4}$ / 75.0$_{\pm 0.0}$ & 29.4$_{\pm 0.3}$ & 39.0$_{\pm 0.6}$ & 37.1$_{\pm 0.5}$ / 36.1$_{\pm 0.0}$ \\
DAPO & 77.2$_{\pm 0.5}$ & 16.7$_{\pm 1.9}$ / 23.3$_{\pm 0.0}$ & 16.7$_{\pm 1.9}$ / 23.3$_{\pm 0.0}$ & 57.5$_{\pm 0.0}$ / 80.0$_{\pm 1.4}$ & 29.0$_{\pm 0.4}$ & 40.4$_{\pm 0.6}$ & 39.6$_{\pm 0.9}$ / 42.2$_{\pm 0.5}$  \\
GMPO & 76.6$_{\pm 0.4}$ & 13.3$_{\pm 0.0}$ / 20.0$_{\pm 1.9}$ & \textbf{20.0}$_{\pm 0.0}$ / \textbf{26.7}$_{\pm 0.0}$ & 55.0$_{\pm 1.4}$ / 82.5$_{\pm 0.0}$ & 30.1$_{\pm 0.3}$ & 38.7$_{\pm 0.5}$ & 39.0$_{\pm 0.4}$ / 43.1$_{\pm 0.6}$  \\
APMPO & \textbf{78.0}$_{\pm 0.3}$ & \textbf{20.0}$_{\pm 0.0}$ / \textbf{30.0}$_{\pm 0.0}$ & 16.7$_{\pm 0.0}$ / \textbf{26.7}$_{\pm 0.0}$ & \textbf{62.5}$_{\pm 0.0}$ / \textbf{85.0}$_{\pm 0.0}$ & \textbf{30.5}$_{\pm 0.2}$ & \textbf{42.4}$_{\pm 0.3}$ & \textbf{41.7}$_{\pm 0.1}$ / \textbf{47.2}$_{\pm 0.0}$  \\
\midrule
\multicolumn{8}{c}{\textit{Qwen2.5-3B-Instruct}} \\
\midrule
Base & 62.0 & 0.0 / 3.3 & 0.0 / 6.7 & 35.0 / 55.0 & 24.3 & 29.1 & 25.1 / 21.7 \\
GRPO & 66.0$_{\pm 0.4}$ & 6.7$_{\pm 0.0}$ / 13.3$_{\pm 1.9}$ & 6.7$_{\pm 0.0}$ / 13.3$_{\pm 0.0}$ & 40.0$_{\pm 0.0}$ / 60.0$_{\pm 1.4}$ & 25.4$_{\pm 0.3}$ & 31.5$_{\pm 0.4}$ & 29.4$_{\pm 0.2}$ / 28.9$_{\pm 1.1}$ \\
DAPO & 67.6$_{\pm 0.4}$ & 6.7$_{\pm 1.9}$ / 16.7$_{\pm 0.0}$ & \textbf{10.0}$_{\pm 0.0}$ / \textbf{20.0}$_{\pm 0.0}$ & 45.0$_{\pm 1.4}$ / 65.0$_{\pm 0.0}$ & 26.8$_{\pm 0.4}$ & 32.6$_{\pm 0.5}$ & 31.5$_{\pm 0.8}$ / 33.9$_{\pm 0.0}$ \\
GMPO & 66.8$_{\pm 0.5}$ & 6.7$_{\pm 0.0}$ / 16.7$_{\pm 0.0}$ & \textbf{10.0}$_{\pm 0.0}$ / 16.7$_{\pm 1.9}$ & 42.5$_{\pm 0.0}$ / 60.0$_{\pm 1.4}$ & 26.1$_{\pm 0.3}$ & 32.2$_{\pm 0.4}$ & 30.7$_{\pm 0.2}$ / 31.1$_{\pm 1.1}$ \\
APMPO & \textbf{68.4}$_{\pm 0.2}$ & \textbf{10.0}$_{\pm 0.0}$ / \textbf{20.0}$_{\pm 0.0}$ & \textbf{10.0}$_{\pm 0.0}$ / \textbf{20.0}$_{\pm 0.0}$ & \textbf{45.0}$_{\pm 0.0}$ / \textbf{70.0}$_{\pm 0.0}$ & \textbf{27.9}$_{\pm 0.2}$ & \textbf{33.2}$_{\pm 0.3}$ & \textbf{32.4}$_{\pm 0.1}$ / \textbf{36.7}$_{\pm 0.0}$ \\
\midrule
\multicolumn{8}{c}{\textit{DeepSeek-R1-Distill-Qwen-1.5B}} \\
\midrule
Base & 64.6 & 6.7 / 26.7 & 13.3 / 40.0 & 47.5 / 70.0 & 24.6 & 30.9 & 31.3 / 45.6 \\
GRPO & 75.4$_{\pm 0.5}$ & 13.3$_{\pm 0.0}$ / 33.3$_{\pm 1.9}$ & 20.0$_{\pm 1.9}$ / 43.3$_{\pm 0.0}$ & 57.5$_{\pm 1.4}$ / 82.5$_{\pm 0.0}$ & 29.8$_{\pm 0.3}$ & 43.2$_{\pm 0.5}$ & 39.9$_{\pm 0.8}$ / 53.0$_{\pm 0.6}$ \\
DAPO & 79.8$_{\pm 0.4}$ & 20.0$_{\pm 1.9}$ / 46.7$_{\pm 0.0}$ & 23.3$_{\pm 0.0}$ / \textbf{50.0}$_{\pm 0.0}$ & 60.0$_{\pm 0.0}$ / 90.0$_{\pm 1.4}$ & 30.1$_{\pm 0.4}$ & 43.8$_{\pm 0.5}$ & 42.8$_{\pm 0.5}$ / 62.2$_{\pm 0.5}$ \\
GMPO & 76.6$_{\pm 0.6}$ & 16.7$_{\pm 0.0}$ / 43.3$_{\pm 1.9}$ & 23.3$_{\pm 1.9}$ / 46.7$_{\pm 0.0}$ & 62.5$_{\pm 1.4}$ / 87.5$_{\pm 0.0}$ & 30.9$_{\pm 0.5}$ & 44.8$_{\pm 0.6}$ & 42.5$_{\pm 0.8}$ / 59.2$_{\pm 0.6}$ \\
APMPO & \textbf{81.6}$_{\pm 0.3}$ & \textbf{23.3}$_{\pm 0.0}$ / \textbf{50.0}$_{\pm 0.0}$ & \textbf{26.7}$_{\pm 0.0}$ / \textbf{50.0}$_{\pm 0.0}$ & \textbf{65.0}$_{\pm 0.0}$ / \textbf{92.5}$_{\pm 0.0}$ & \textbf{32.7}$_{\pm 0.2}$ & \textbf{46.6}$_{\pm 0.4}$ & \textbf{46.0}$_{\pm 0.2}$ / \textbf{64.2}$_{\pm 0.0}$ \\
\bottomrule
\end{tabular}
\caption{Experimental results on multiple mathematical reasoning benchmarks. The results are reported as mean and standard deviation across 3 random seeds (format: $\text{Mean}_{\pm \text{Std}}$). The best results are highlighted in bold. Note that ``$\cdot/\cdot$'' indicates ``Pass@1/Pass@16'', and the last column indicates ``Average Pass@1 / Average Pass@16''.}
\label{tab:exp}
\end{table*}
\section{Experiments}\label{sec:exp}

In this study, we conducted a systematic evaluation to address the following research questions:

\textbf{RQ1:} How does APMPO compare to state-of-the-art RLVR-based baselines? \textbf{RQ2:} Does APMPO achieve superior training performance than previous methods? \textbf{RQ3:} What are the contributions of PMPO and FAC to overall performance? \textbf{RQ4:} How sensitive is APMPO to different values of $\gamma$, $\epsilon_{\text{min}}$, and $\epsilon_{\text{max}}$?

\subsection{Settings}

\noindent \textbf{Models.} In this work, Qwen2.5-Math-1.5B-Instruct \cite{yang2024qwen21}, Qwen2.5-3B-Instruct \cite{yang2024qwen2}, and DeepSeek-R1-Distill-Qwen-1.5B \cite{guo2025deepseek} were employed for mathematical reasoning. For SQL generation and multi-modal reasoning, we utilized Qwen2.5-Coder-3B-Instruct \cite{hui2024qwen2} and Qwen2.5-VL-3B-Instruct \cite{bai2025qwen2}, respectively.

\noindent \textbf{Datasets.} For mathematical reasoning, training was performed on the MATH dataset \cite{hendrycks2021measuring}. Moreover, MATH500 \cite{hendrycks2021measuring}, AIME24 \cite{li2024numinamath}, AIME25 \cite{codeforcesamerican}, AMC23 \cite{ouyang2022training}, Minerva \cite{lewkowycz2022solving}, and OlympiadBench \cite{huang2024olympicarena} were used for evaluation. For SQL generation, BIRD-Train \cite{li2024can} served as the training set, and we evaluated on Spider-Dev \cite{yu2018spider} and BIRD-Dev \cite{li2024can}. For multi-modal reasoning, we used Geometry3K \cite{lu2021inter}, which included dedicated training and test subsets. Further details of models and datasets are given in Appendix \ref{appendix:reference}.

\noindent \textbf{Implementation Details.}
For all experiments, we employed a binary reward function where each response received a reward of $1$ if it was correct and $0$ otherwise. During training, the coefficient of KL loss term was $\beta=0.001$. The batch size and the number of rollouts were $512$ and $8$, and we used a sampling temperature of $1.0$. We used the AdamW optimizer \cite{zhou2024towards} with a learning rate of $1 \times 10^{-6}$ and trained for $400$ steps. For APMPO, $\gamma$, $\epsilon_\text{min}$, $\epsilon_\text{max}$, $\epsilon_{\text{low}}$, and $\delta$ were fixed at $0.8$, $0.2$, $0.4$, $0.2$, and $1 \times 10^{-6}$, respectively. During evaluation, the sampling temperature was set to $0.6$, and Pass@1 was used as the primary metric. Pass@16 was also reported for small-sized datasets (\textit{i.e.}, AIME24, AIME25, and AMC23). For SQL generation, the reward was based on execution accuracy. All experiments were conducted on four NVIDIA GeForce A100 40GB GPUs.

\begin{figure*}[t]
\centering
\subfigure[Experimental results on Spider-Dev and BIRD-Dev by using different RL-based methods.]{
		\includegraphics[scale=0.35]{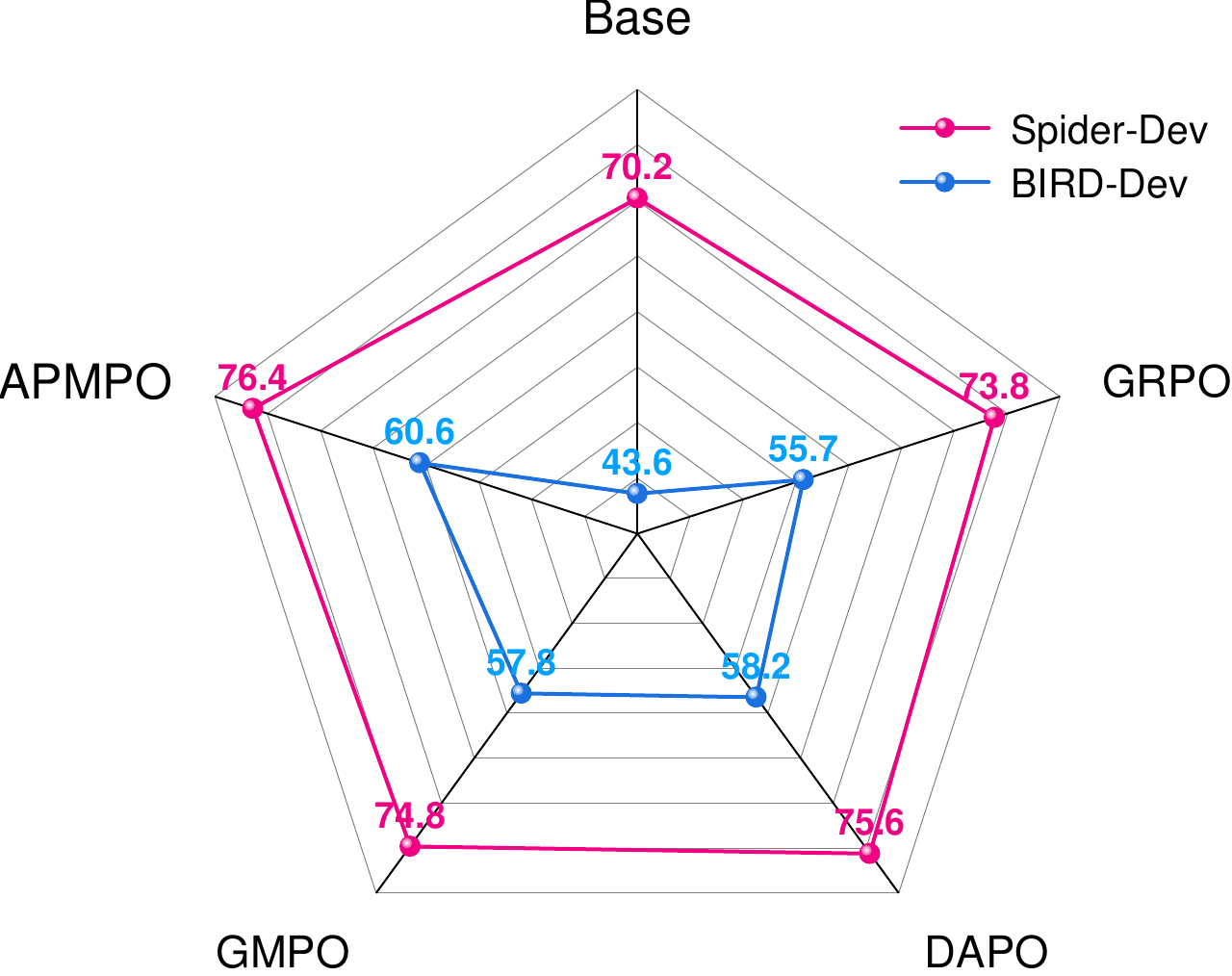}}
\subfigure[Experimental results on Geometry3K by using different RL-based methods.]{
		\includegraphics[scale=0.35]{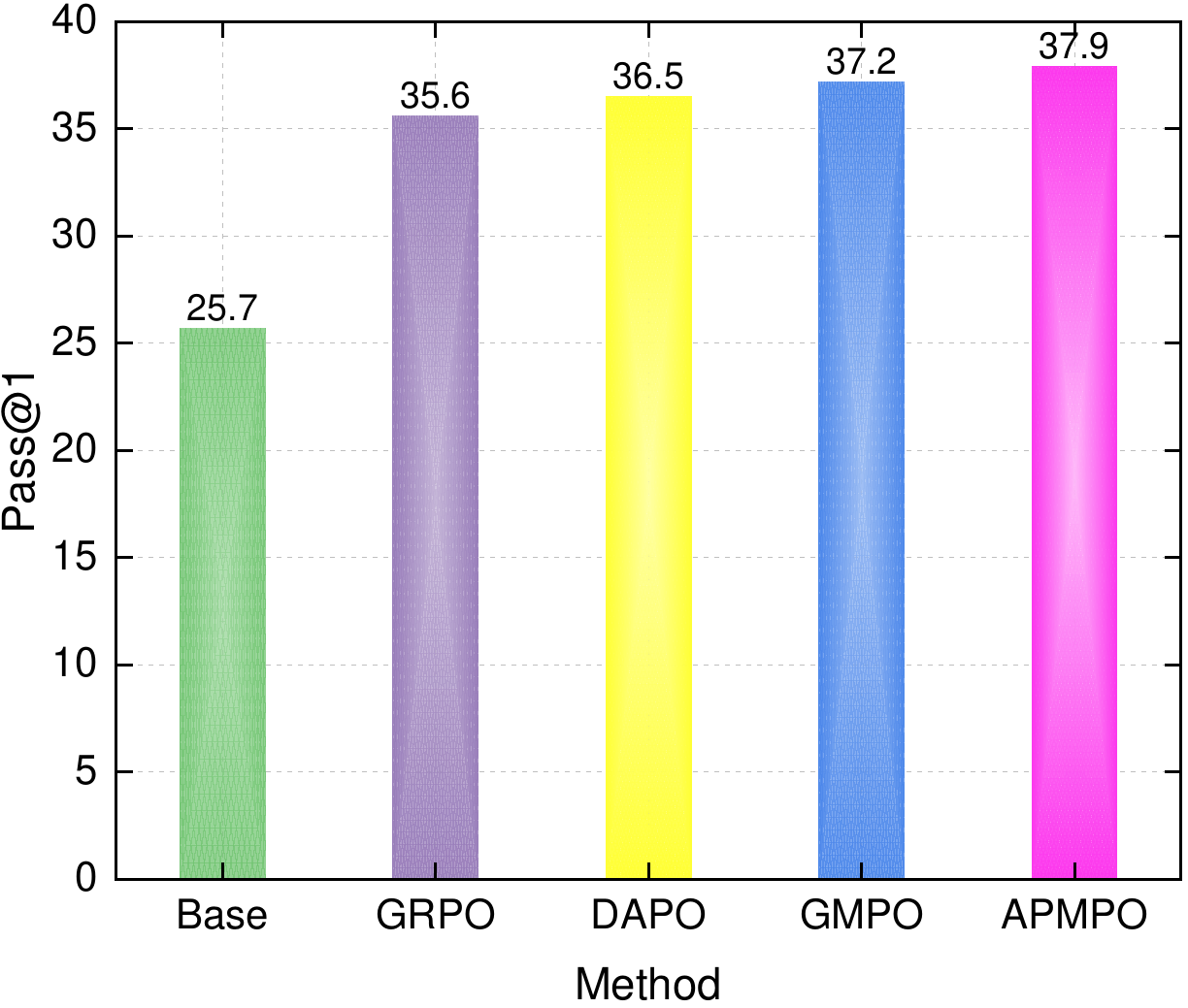}}
\caption{Experimental results on (a) SQL generation, and (b) multi-modal reasoning tasks.}
\label{fig:geoSQL}
\end{figure*}

\subsection{Experimental Results}

\subsubsection{Main Results (RQ1)}

As presented in Table \ref{tab:exp}, APMPO consistently outperforms all baselines across mathematical reasoning benchmarks. For instance, on Qwen2.5-Math-1.5B-Instruct, APMPO achieves an average score of 41.7, surpassing the strongest baseline (DAPO) by 2.1 points. This consistent superiority across diverse model sizes and architectures confirms the superiority of APMPO. Additionally, we evaluate APMPO on SQL generation and multi-modal reasoning to showcase its broad applicability. The results in Figure \ref{fig:geoSQL} show consistent improvements in SQL generation and multi-modal reasoning tasks. For instance, APMPO achieves the highest Pass@1 scores of 75.4 and 57.6 on Spider-Dev and BIRD-Dev, respectively. Further results are shown in Appendix \ref{appendix:geoSQL}.

\subsubsection{Analysis of Training Dynamics (RQ2)}

\noindent \textbf{Training Rewards and Entropy.} Figure~\ref{fig:training}(a) indicates that APMPO (red curve) achieves higher training rewards than all baselines, which correlates with the policy entropy trends in Figure~\ref{fig:training}(b). While other methods exhibit a notable entropy collapse, APMPO maintains higher entropy throughout the majority of training. This sustained entropy signifies better regularization, allowing the model to explore diverse reasoning paths before converging. In essence, APMPO alleviates the early collapse of GRPO and and mitigates GMPO’s overly conservative exploration, leading to stronger final performance. Further analysis on the output diversity of APMPO is shown in Appendix \ref{sec:diversity_analysis}.

\noindent \textbf{Training Efficiency.} 
Figure~\ref{fig:training}(c) presents the computational efficiency of all compared methods on Qwen2.5‑Math-1.5B-Instruct. The adaptive mechanisms in APMPO are computationally lightweight, where computing $p$ and $\epsilon_{\text{ada}}$ requires only batch-level statistics (\textit{i.e.}, mean and standard deviation). As a result, APMPO incurs negligible computational overhead while delivering substantial performance gains. A further analysis on training efficiency is provided in Appendix \ref{appendix:complexity_analysis}.

\subsubsection{Ablation Study (RQ3)}

\begin{figure*}[h] 
	\centering  
	\subfigure[Experimental results using different components of APMPO.]{
		\includegraphics[width=0.31\linewidth]{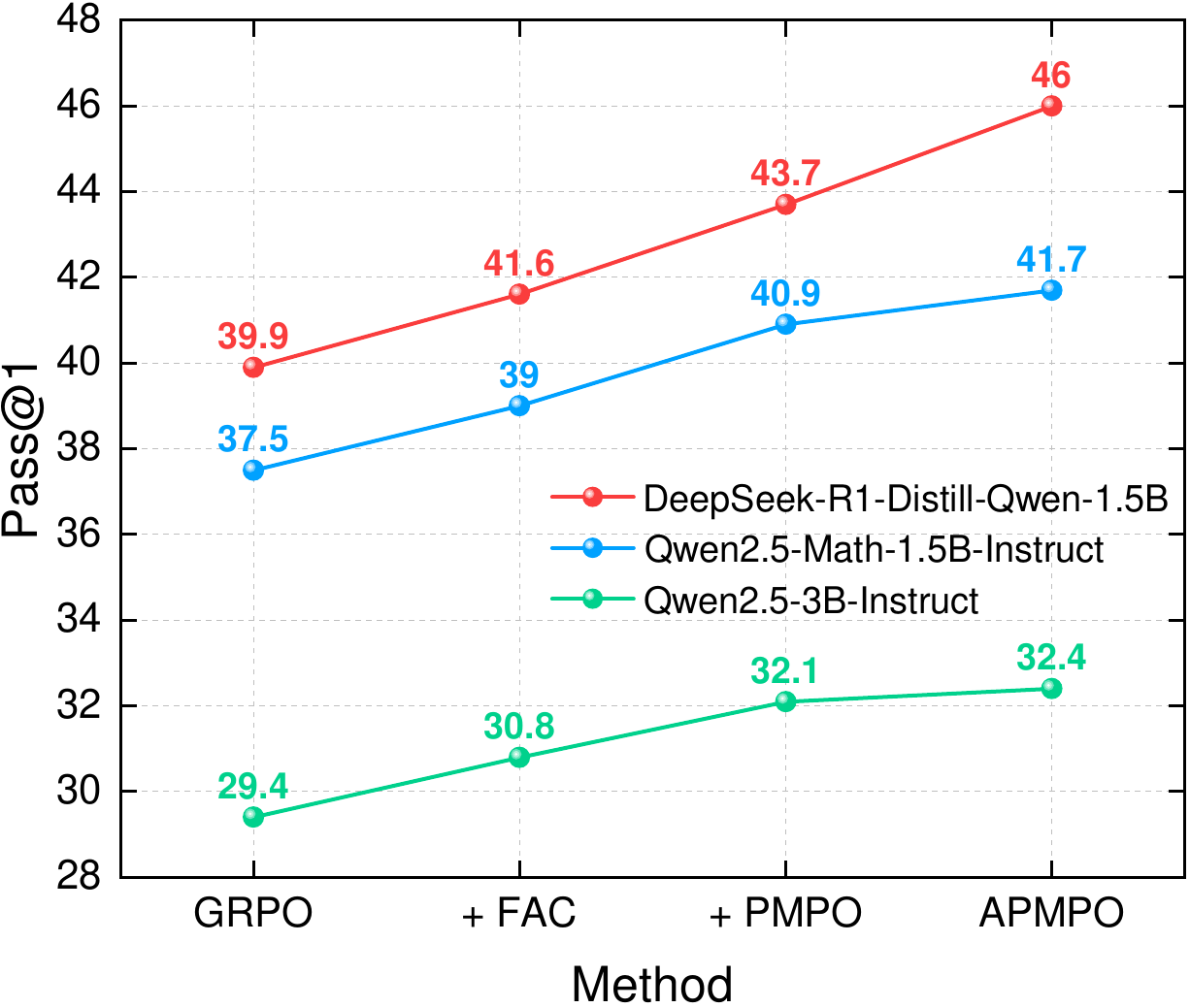}}
	\quad 
	\subfigure[Experimental results using different variants of FSS.]{
		\includegraphics[width=0.31\linewidth]{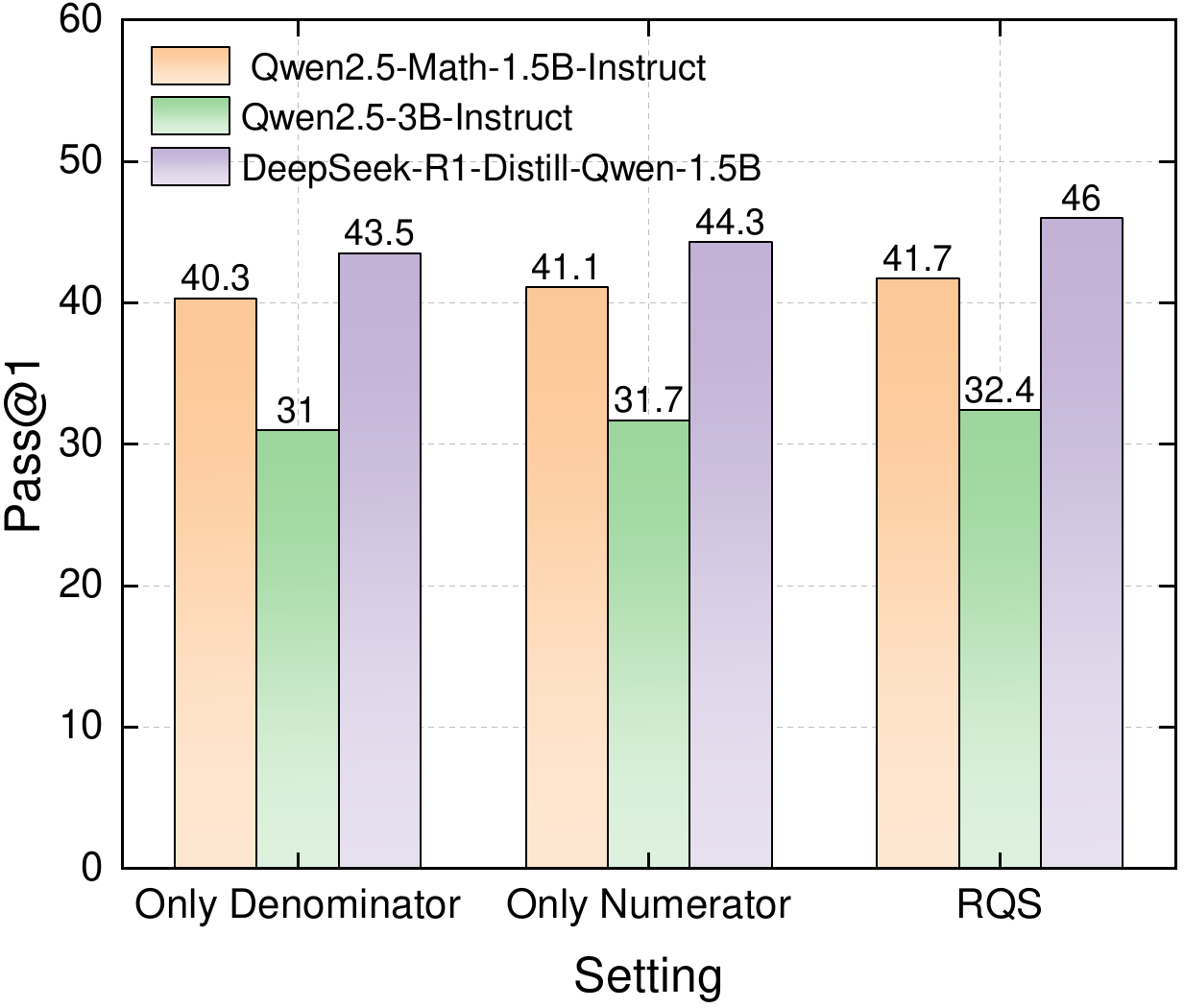}}
    \subfigure[Experimental results using different variants of adaptive exponent function.]{
		\includegraphics[width=0.31\linewidth]{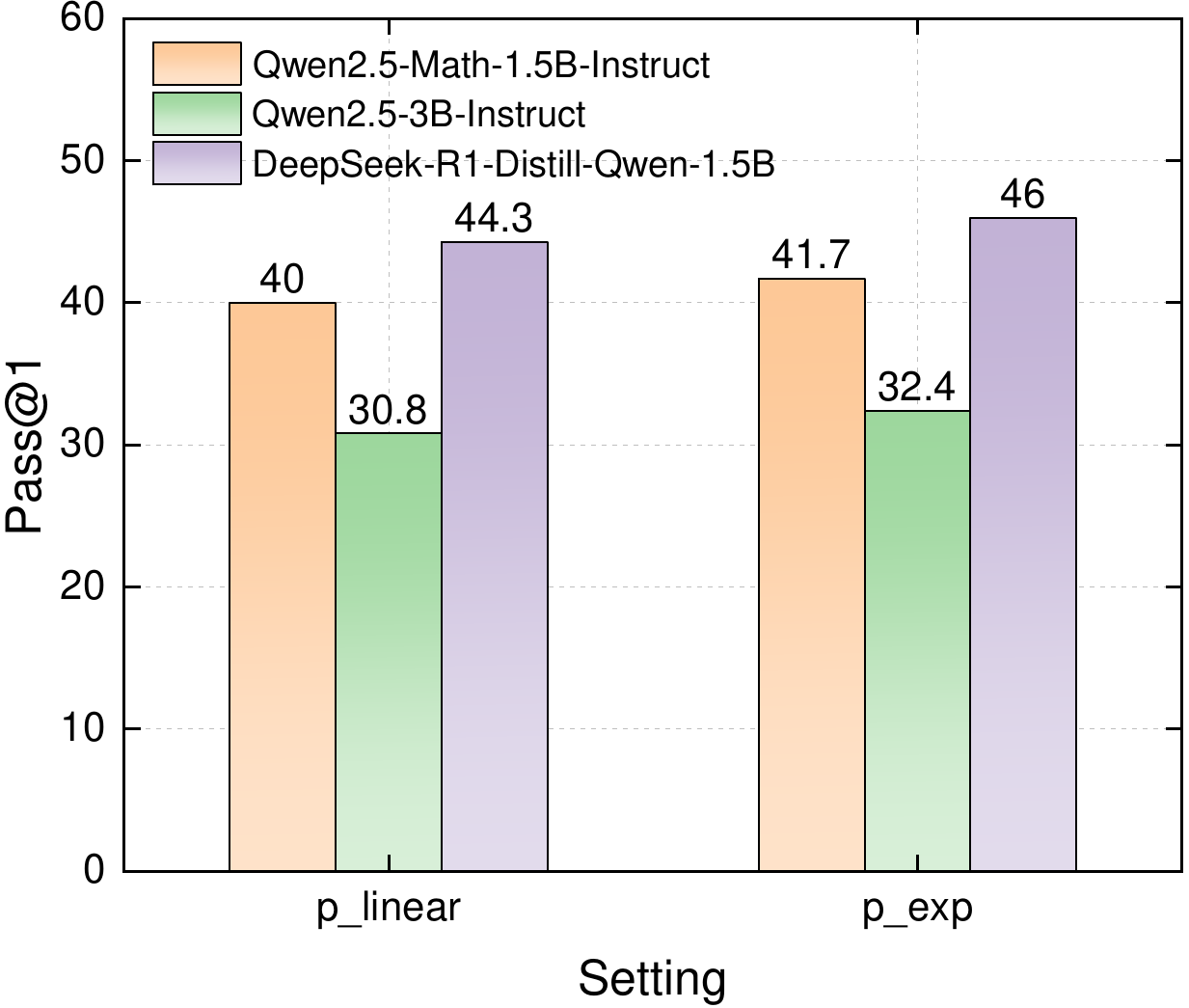}}
	\subfigure[Experimental results using different values of $\gamma$.]{
		\includegraphics[width=0.31\linewidth]{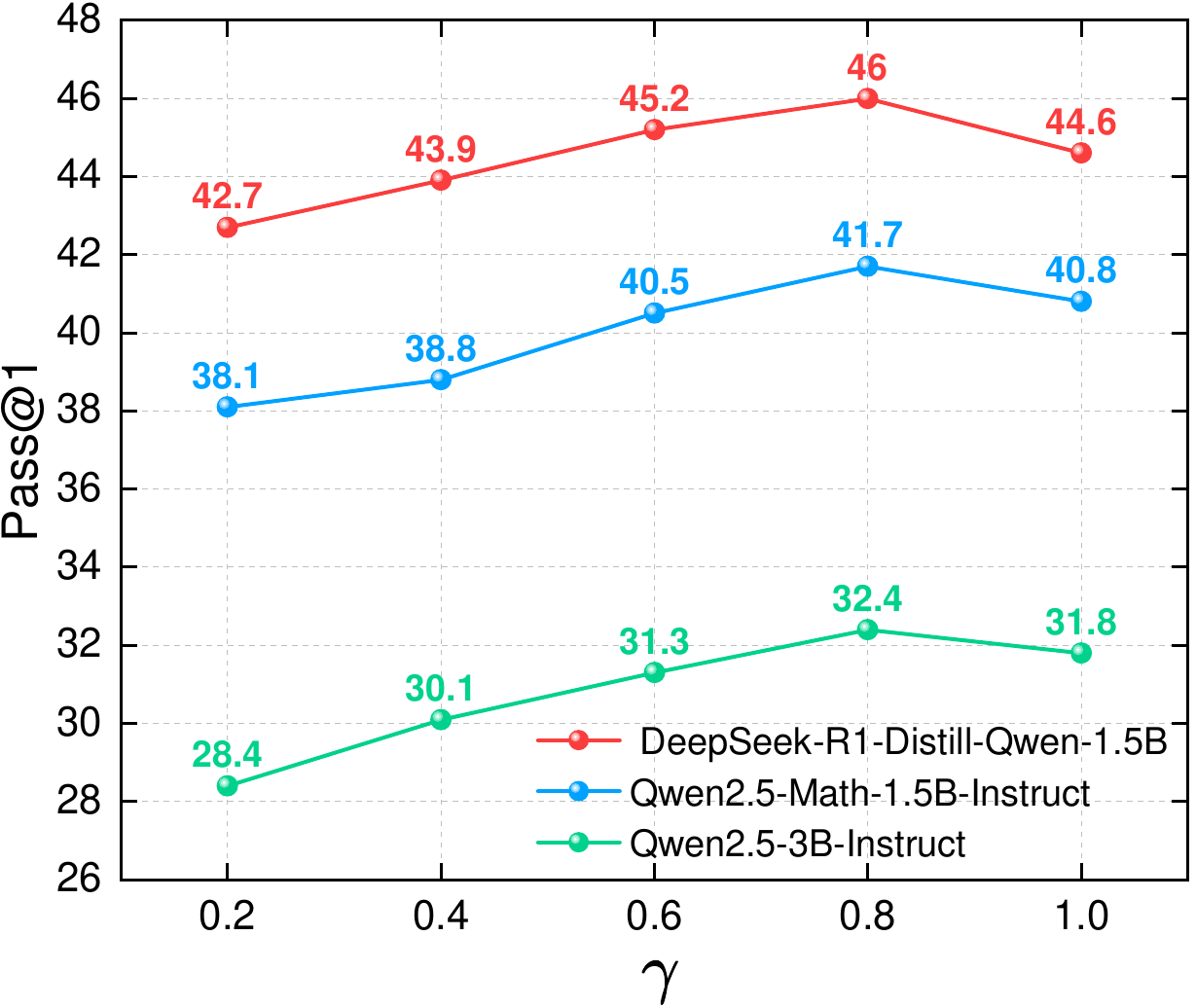}}
	\quad
	\subfigure[Experimental results using different values of $(\epsilon_\text{min},\epsilon_\text{max})$.]{
		\includegraphics[width=0.31\linewidth]{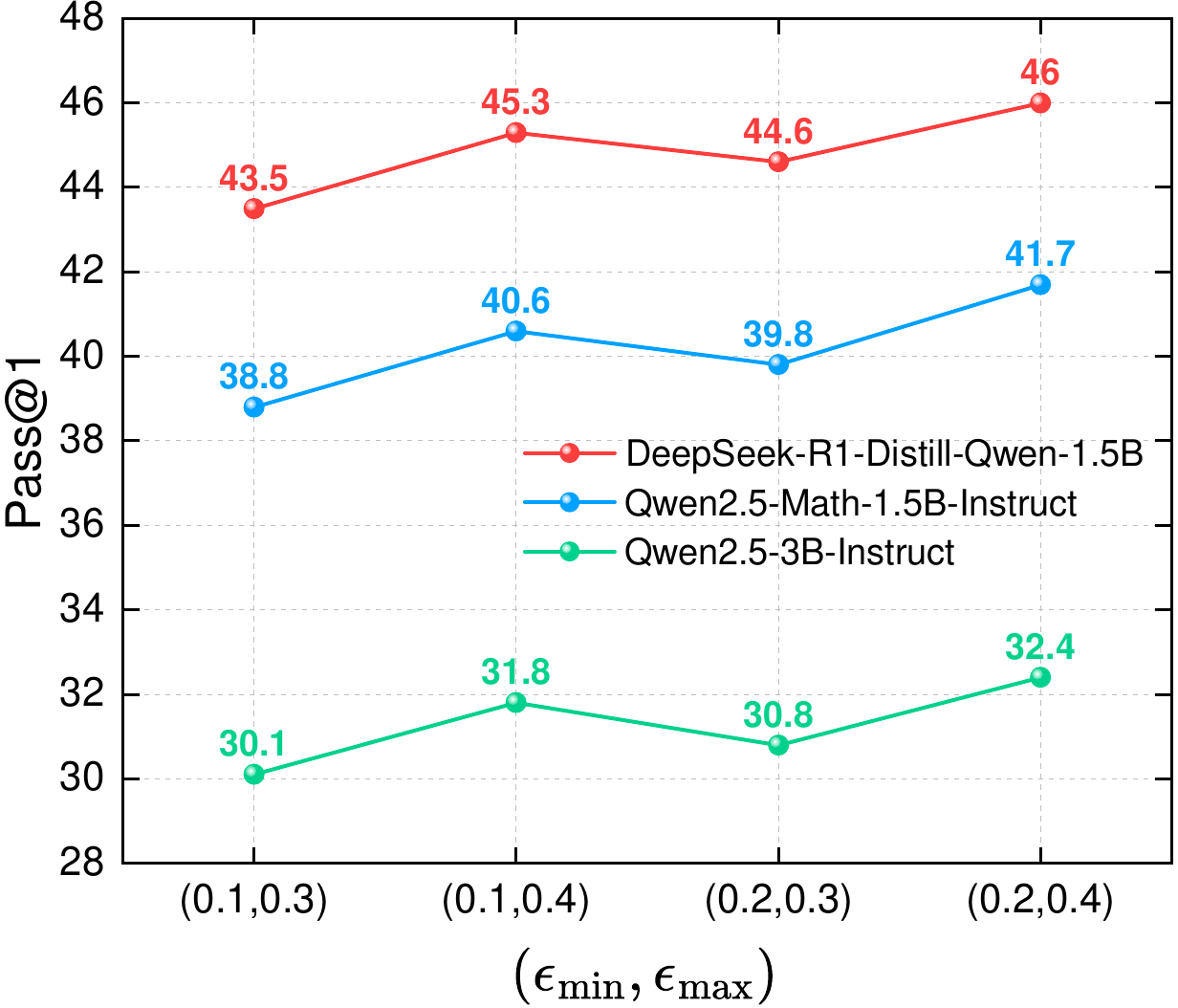}}
	\caption{Ablation studies and sensitivity analysis. Results are reported as average Pass@1 scores on mathematical reasoning benchmarks. (a) Contribution of each component (PMPO and FAC). (b-c) Effectiveness of different design choices for FSS and the adaptive exponent function. (d-e) Sensitivity analysis for $\gamma$ and $(\epsilon_\text{min},\epsilon_\text{max})$.}
	\label{fig:abl}
\end{figure*}

\noindent \textbf{Efficacy of PMPO and FAC.} As shown in Figure \ref{fig:abl}(a), an ablation study is performed on mathematical reasoning to quantify the contribution of each component. Starting from GRPO, incorporating FAC alone yields consistent performance gains, indicating that adaptively adjusting clipping bounds benefited training. In contrast, applying PMPO alone leads to larger improvements, suggesting that handling the exploration-stability trade-off via an adaptive objective is the primary source of these gains. Combining PMPO and FAC achieves the highest scores across all benchmarks, thereby showcasing a clear synergistic effect. A more in-depth analysis and additional results are shown in Appendices \ref{appendix:synergy} and \ref{appen:component}, respectively.

\noindent \textbf{Ablation of FSS Components.} The full FSS is compared with variants using only the numerator ($\mu_{R}$) and the denominator ($1/\sigma_{R}$). Results in Figure \ref{fig:abl}(b) reveal the superiority of the complete FSS. Specifically, relying solely on $\mu_{R}$ offers limited sensitivity to variations in reward stability. Conversely, dependence only on $1/\sigma_{R}$ favors batches where reward signals show minimal variation, potentially promoting consistent failure. These findings indicate that the efficacy of FSS stems from its specific formulation, which amplifies high-confidence successful signals while suppressing unreliable or consistently incorrect outcomes. Further analyses of FSS and additional results are shown in Appendices \ref{appendix:FSS} and \ref{appen:FSS}, respectively.

\noindent \textbf{Ablation of Adaptive Exponent Formulation.} The exponential-decay formulation for the adaptive exponent $p$ (Eq. (\ref{eq:dynamic_p})) is compared against a simplified linear-decay variant defined as:
\begin{equation}
   p_\text{linear}=1-\gamma \cdot \mu_R
\end{equation}

Figure~\ref{fig:abl}(c) shows that the exponential‑decay formulation outperforms the linear‑decay variant. This result suggests that the efficacy of PMPO stems from its smooth and asymptotic transition, which avoids abrupt shifts in the learning objective and supports stable policy updates. Further analysis and additional results are provided in Appendices \ref{appendix:p} and \ref{appendix:exponent}, respectively.

\subsubsection{Sensitivity Analysis (RQ4)}
\noindent \textbf{$\gamma$ in PMPO.} 
A sensitivity analysis is performed to assess the impact of $\gamma$ in PMPO. We vary $\gamma$ and report average Pass@1 scores on mathematical reasoning datasets. As shown in Figure \ref{fig:abl}(d), a small $\gamma$ induces a slow transition, where the objective remains GRPO-like for an extended period. Under this setting, heightened sensitivity to reward outliers leads to constrained exploration. Conversely, a large $\gamma$ induces a fast shift to a GMPO-like objective, thereby excessively suppressing high-reward outliers and risking early convergence to suboptimal solutions. Based on the results, $\gamma=0.8$ is adopted in this study. More experimental results are provided in Appendix \ref{appen:gamma}.

\noindent \textbf{$(\epsilon_\text{min}, \epsilon_\text{max})$ in FAC.} We further assess the reasoning performance of APMPO under variations in the clipping bounds ($\epsilon_{\text{min}}$, $\epsilon_{\text{max}}$). Figure \ref{fig:abl}(e) illustrates Pass@1 scores on mathematical reasoning datasets. While APMPO reaches its peak performance using the clipping bound $(0.2, 0.4)$, other tested alternatives produces comparable results, including the wider range of $(0.1,0.4)$ and the narrower range of $(0.2,0.3)$. These findings demonstrate that APMPO remains robust to the choice of clipping bounds. In this work, $(0.2, 0.4)$ is selected, and more results are provided in Appendix \ref{appen:epsilon}.

\section{Conclusion}\label{sec:con}

This work introduces APMPO, which is designed to enhance the reasoning capabilities of LLMs from an adaptive perspective. To overcome the limitations of static objective functions and rigid clipping strategies in previous methods, APMPO incorporates Power-Mean Policy Optimization (PMPO) and Feedback-Adaptive Clipping (FAC). Extensive experiments on several reasoning benchmarks validate the efficacy of APMPO, showcasing a superior learning process compared to state-of-the-art RLVR-based baselines.

\section*{Limitations}

Despite the promising results achieved by APMPO, there are several limitations that warrant further attention. (1) Due to limited computational resources, the experiments were constrained to models with 1.5B and 3B parameters. However, the observed results consistently demonstrate the effectiveness of APMPO across these scales. Meanwhile, it is expected that the proposed adaptive mechanisms would yield comparable or even greater benefits when applied to larger models exhibiting more complex emergent behaviors. (2) While our experiments span multiple domains, the core of the method depends on the availability of accurate and verifiable outcome‑based rewards. The applicability of APMPO to domains where such reward signals are difficult to define remains a promising direction for future work.

\section*{Ethical Considerations}
We provide the following ethical statements: (1) The efficacy of APMPO was evaluated using open‑source LLMs. These models require substantial computational resources, which may contribute to increased carbon dioxide emissions and high energy consumption. (2) All fine‑tuned models in this work are derived from publicly released, open‑source architectures. No proprietary or confidential models were used. (3) The datasets used for training and evaluation are publicly available, ensuring that no data privacy concerns arise.

\section*{Acknowledgments} 
This work is supported by the National Key Research and Development Program of China under Grant 2023YFB3106504, the National Natural Science Foundation of China under project (No. 62472126), Shenzhen Science and Technology Program under Grant ZDSYS20210623091809029, the Major Key Project of PCL under Grant PCL2024A04 and PCL2025A16, and CCF-Huawei Populus Grove Fund.

\bibliography{custom}

\newpage
\appendix

\newpage
\section{Appendix A: Description of the Compared Methods} \label{appendix:compared}

This section provides a formal description of the policy optimization methods compared in our work. All methods estimate advantage based on a group of sampled rollouts, thereby eliminating the need for a separate value model.

\subsection{GRPO}
GRPO is the foundational method, which first samples $G$ responses for each query, assigns rewards $R$ via a rule-based reward function, and estimates the corresponding advantages as in Eq. (\ref{eq:group_advantage}). Finally, the model parameters are updated as follows:
\begin{equation}
    \begin{aligned}
    \mathcal{J}(\theta) &= 
    \frac{1}{G} \sum_{i=1}^G \frac{1}{|o_i|} \sum_{t=1}^{|o_i|} 
    \min\Big( r_{i,t} \hat{A}_{i},\,  \\
    &\text{clip}(r_{i,t}, 1-\epsilon, 1+\epsilon) \hat{A}_{i} \Big)
    \end{aligned}
    \label{eq:GRPO}
\end{equation}
where $\epsilon$ was fixed at $0.2$ in our implementation.

\subsection{DAPO}
The primary innovations in DAPO lie in advanced clipping mechanism and the Token Level Mean loss (TLM). Specifically, DAPO decouples the clipping range into asymmetric bounds, which focuses on enhancing exploration and mitigating entropy collapse. Moreover, since longer sequences exert a disproportionate influence on gradient update in GRPO, DAPO standardizes token counts across the entire batch. Collectively, the objective function of DAPO is formally given as:
\begin{equation}
    \begin{aligned}
    \mathcal{J}(\theta) &= 
    \frac{1}{\sum_{i=1}^G |o_i|} \sum_{i=1}^G \sum_{t=1}^{|o_i|} 
    \min\Big( r_{i,t} \hat{A}_{i},\, \\ 
    &\text{clip}( r_{i,t}, 1-\epsilon_{\text{low}}, 1+\epsilon_{\text{high}}) \hat{A}_{i} \Big)
    \end{aligned}
    \label{eq:DAPO}
\end{equation}
where the upper clipping bound $\epsilon_{\text{high}}$ is set higher than the lower bound $\epsilon_{\text{low}}$. Following the original work \cite{yu2025dapo}, $\epsilon_{\text{low}}$ and $\epsilon_{\text{high}}$ were set to $0.2$ and $0.28$, respectively.

\subsection{GMPO}
GMPO addresses the training instability caused by outlier rewards by changing the optimization objective. Instead of maximizing the arithmetic mean of token-level importance-weighted rewards, GMPO maximizes their geometric mean. While the advantage calculation remains the same as in Eq. (\ref{eq:group_advantage}), the objective function for a single response $o_i$ is conceptually structured to optimize for the geometric mean of importance sampling ratios:
\begin{equation}
    \begin{aligned}
    \mathcal{J}(\theta) &= 
    \frac{1}{G} \sum_{i=1}^G \Bigg\{ \prod_{t=1}^{|o_i|} 
    \min\Big( r_{i,t}^{\text{sgn}(\hat{A}_i)},  \\
    &\text{clip}(r_{i,t}^{\text{sgn}(\hat{A}_i)}, \epsilon_1, \epsilon_2) \Big) \Bigg\}^{\frac{1}{|o_i|}} \hat{A}_{i}
    \end{aligned}
    \label{eq:GMPO}
\end{equation}
where $\epsilon_1$ and $\epsilon_2$ were set to $e^{-0.4}$ and $e^{0.4}$, respectively. This structural change provides more stable policy updates in the presence of reward outliers. 

\section{Appendix B: Details of Models and Benchmarks} \label{appendix:reference}

In the following, the links regarding the models and benchmarks used in this work are provided. Notably, all datasets adopted in this study are publicly available under the CC BY-SA 4.0 license.

(1) \textbf{Models:}
\begin{itemize}
    \item Qwen2.5-Math-1.5B-Instruct:\\
    \url{https://huggingface.co/Qwen/Qwen2.5-Math-1.5B-Instruct}
    \item Qwen2.5-3B-Instruct:\\
    \url{https://huggingface.co/Qwen/Qwen2.5-3B-Instruct}
    \item DeepSeek-R1-Distill-Qwen-1.5B:\\\url{https://huggingface.co/deepseek-ai/DeepSeek-R1-Distill-Qwen-1.5B}
    \item Qwen2.5-Coder-3B-Instruct:\\
    \url{https://huggingface.co/Qwen/Qwen2.5-Coder-3B-Instruct}
    \item Qwen2.5-VL-3B-Instruct:\\
    \url{https://huggingface.co/Qwen/Qwen2.5-VL-3B-Instruct}
\end{itemize}

(2) \textbf{Benchmarks:}
\begin{itemize}
    \item MATH: (12000 training samples)\\
    \url{https://huggingface.co/datasets/HuggingFaceH4/MATH}
    \item MATH500: (500 evaluation samples)\\
    \url{https://huggingface.co/datasets/HuggingFaceH4/MATH-500}
    \item AIME24: (30 evaluation samples)\\
    \url{https://huggingface.co/datasets/HuggingFaceH4/aime_2024}
    \item AIME25: (30 evaluation samples)\\ 
    \url{https://huggingface.co/datasets/HuggingFaceH4/aime_2025}
    \item AMC23: (40 evaluation samples)\\
    \url{https://huggingface.co/datasets/math-ai/amc23}
    \item Minerva: (272 evaluation samples)\\
    \url{https://huggingface.co/datasets/svc-huggingface/minerva-math}
    \item OlympiadBench: (674 evaluation samples)\\
    \url{https://huggingface.co/datasets/knoveleng/OlympiadBench}
    \item Spider: (1034 evaluation samples)\\
    \url{https://yale-lily.github.io/spider}
    \item BIRD: (9428 and 1534 samples for training and evaluation\\
    \url{https://bird-bench.github.io/}
    \item Geometry3K: (2100 and 601 samples for training and evaluation) \\
    \url{https://huggingface.co/datasets/hiyouga/geometry3k}
\end{itemize}

\section{Appendix C: Gradient Analysis of GRPO and GMPO} \label{sec:appendix_example}

\newtheorem{theorem}{\bf Theorem}
\newtheorem{definition}{\bf Definition}
\newtheorem{proposition}{\bf Proposition}
\newtheorem{lemma}{\bf Lemma}
\newtheorem{argument}{\bf Argument}
\newtheorem{conclusion}{\bf Conclusion}
\newtheorem{Proof}{\bf Proof}
\newtheorem{assumption}{\bf Assumption}
\newtheorem{remark}{\bf Remark}
\newtheorem{corollary}{\bf Corollary}

\newcommand{\adv}{\hat{A}_i}
\newcommand{\absadv}{|\hat{A}_i|}
\newcommand{\rhoo}{\rho_{i,t}(\theta)}
\newcommand{\psii}{\psi_{i,t}(\theta)}
\newcommand{\oldpi}{\pi_{\text{old}}}
\newcommand{\newpi}{\pi_{\theta}}

In this section, we provide a gradient derivation for GRPO and GMPO. We omit clipping mechanisms and KL penalties to isolate the impact of each objective function on the learning signal, specifically regarding sensitivity to high-reward outliers.

\subsection{Preliminaries}
Let $G$ denote the group size. For the $i$-th sample $o_i$ sampled from the group ($i \in \{1, \dots, G\}$), let $|o_i|$ be its token length. We define $r_{i,t}(\theta) = \frac{\newpi(o_{i,t}|o_{i,<t})}{\oldpi(o_{i,t}|o_{i,<t})}$ and $\psii = \nabla_\theta \log \newpi(o_{i,t}|o_{i,<t})$.

\subsection{Gradient Derivation of GRPO (Arithmetic Mean)}

\begin{definition}[\textbf{GRPO Objective}]
The objective function of GRPO based on arithmetic mean is defined as:
\begin{equation}
    \mathcal{J}_{\text{GRPO}}(\theta) = \frac{1}{G} \sum_{i=1}^G \left( \frac{1}{|o_i|} \sum_{t=1}^{|o_i|} r_{i,t}(\theta) \right) \cdot \adv 
\end{equation}
\end{definition}

\begin{proposition}[\textbf{Linear Scaling of GRPO}]
The gradient of $\mathcal{J}_{\text{GRPO}}$ scales linearly with the advantage value $\adv$.
\end{proposition}

\begin{proof}
Differentiating $\mathcal{J}_{\text{GRPO}}$ with respect to $\theta$:
\begin{equation}
    \begin{aligned}
    \nabla_\theta \mathcal{J}_{\text{GRPO}} &= \frac{1}{G} \sum_{i=1}^G \frac{\adv}{|o_i|} \sum_{t=1}^{|o_i|} \nabla_\theta r_{i,t}(\theta) \\
    &= \frac{1}{G} \sum_{i=1}^G \frac{\adv}{|o_i|} \sum_{t=1}^{|o_i|} \frac{\nabla_\theta \newpi(\cdot)}{\oldpi(\cdot)}
\end{aligned}
\end{equation}

Using the identity $\frac{\nabla \pi}{\pi_{old}} = \frac{\pi}{\pi_{old}} \nabla \log \pi = r_{i,t}(\theta) \psii$ yields:
\begin{equation}
    \nabla_\theta \mathcal{J}_{\text{GRPO}} \approx \frac{\hat{A}_i}{|o_i|} \sum_{t=1}^{|o_i|} \underbrace{r_{i,t}(\theta)}_{\text{Local Weight}} \cdot \psii
\end{equation}

Therefore, the gradient magnitude is directly proportional to $\adv$, and outliers with large $\adv$ disproportionately dominate the gradient update.
\end{proof}

\subsection{Gradient Derivation of GMPO (Geometric Mean)}

\begin{definition}[\textbf{GMPO Objective}]
GMPO calculates the geometric mean of the importance sampling ratios over the sequence length, scaled by the advantage. This formulation fundamentally alters the weighting mechanism of individual token gradients compared to GRPO.
\begin{equation}
    \mathcal{J}_{\text{GMPO}}(\theta) = \frac{1}{G} \sum_{i=1}^G \underbrace{\left\{ \prod_{t=1}^{|o_i|} r_{i,t}(\theta) \right\}^{\frac{1}{|o_i|}}}_{U_i(\theta)} \cdot \hat{A}_i
\end{equation}
\end{definition}

\begin{proposition}[\textbf{Global vs. Local Gradient Weighting}]
The gradient derivations reveal a structural divergence in weighting mechanisms. While GRPO assigns a \textit{local weight} (\textit{i.e.}, $r_{i,t}$) to each gradient step, GMPO assigns a \textit{global weight} (\textit{i.e.}, $U_i$) to every token in the sequence.
\end{proposition}

\begin{proof}
We analyze the gradient contribution of a single sample $i$. Applying the log-derivative trick $\nabla f = f \nabla \log f$:
\begin{equation}
    \begin{aligned}
    \nabla_\theta \mathcal{J}_i(\theta) &= \hat{A}_i \cdot \nabla_\theta U_i(\theta) \\
    &= \hat{A}_i \cdot U_i(\theta) \cdot \nabla_\theta \left( \frac{1}{|o_i|} \sum_{t=1}^{|o_i|} \log r_{i,t}(\theta) \right) \\
    &= \frac{\hat{A}_i}{|o_i|} \sum_{t=1}^{|o_i|} \underbrace{U_i(\theta)}_{\text{Global Weight}} \cdot \psii
\end{aligned}
\end{equation}
\end{proof}

\begin{corollary}[\textbf{Holistic Stability vs. Local Aggressiveness}]
The disparity in coefficient terms between GRPO and GMPO aligns with the empirical analysis in Figures \ref{fig:training}(a) and (b):

\begin{itemize}
    \item \textbf{GRPO (Local Sensitivity):} The weight for token $t$ corresponds directly to its individual ratio $r_{i,t}(\theta)$, implying that an extreme ratio at a single token induces a disproportionately large update for that specific token. Consequently, GRPO becomes highly sensitive to local fluctuations and prone to high variance.
    
    \item \textbf{GMPO (Holistic View):} The weight for token $t$ is $U_i(\theta)$, representing the geometric mean of \textit{all} ratios in the sequence. This mechanism imposes a holistic perspective where the update strength is determined by the consistency of the entire generation path. Even if a single token exhibits an extreme ratio, the geometric mean dampens its influence to promote stability. While GMPO effectively suppresses outliers, it inevitably restricts the model's capacity to reinforce pivotal token-level breakthroughs.
\end{itemize}

This analytical contrast motivates a dynamic exponent $p$ in PMPO to switch between local sensitivity and holistic stability depending on training phase.
\end{corollary}

\section{Appendix D: Theoretical Analysis on PMPO} \label{appendix:theo1}

In this section, we provide a theoretical analysis of the PMPO objective. Remarkably, PMPO establishes a unified framework connecting the arithmetic mean-based objective (\textit{e.g.}, GRPO) and the geometric mean-based objective (\textit{e.g.}, GMPO). This allows for dynamic adjustment of policy optimization strategy by varying the exponent $p$. 

\subsection{Preliminaries}

We begin by formally establishing the necessary mathematical definitions.

\begin{definition}[Power Mean]
For a set of non-negative real numbers $X=\{x_1, x_2, \dots, x_n\}$, the generalized power mean with exponent $p \in \mathbb{R}\textbackslash \{0\}$ is defined as:
\begin{equation}
    M_p(X) = \left( \frac{1}{n} \sum_{t=1}^{n} x_t^p \right)^{\frac{1}{p}}
    \label{equ:power}
\end{equation}

The cases for $p \to 1$ and $p \to 0$ are of special interest:
\begin{itemize}
    \item The \textit{Arithmetic Mean} is defined as:
    \begin{equation}
        M_1(X) = \frac{1}{n}\sum_{t=1}^{n} x_t
    \end{equation}
    \item The \textit{Geometric Mean} is defined as:
    \begin{equation}
        M_0(X) = \lim_{p \to 0} M_p(X) = \left( \prod_{t=1}^{n} x_t \right)^{\frac{1}{n}}
    \end{equation}
\end{itemize}
\end{definition}

\begin{definition}[Per-Sequence PMPO Objective]
The full per-sequence PMPO objective for a given response $o_i$ is defined by:
\begin{equation}
\mathcal{J}_i(\theta) = sgn(\hat{A}_i) \cdot M_p(\Phi_i(\theta))
\end{equation}
where $\Phi_i(\theta) = \{\phi_{i,1}(\theta), ..., \phi_{i,|o_i|}(\theta)\}, \phi_{i,t}(\theta) \ge 0$ is the set of token-level magnitudes. Each $\phi_{i,t}(\theta)$ is defined in Eq. (\ref{eq:token_magnitude}). 
\end{definition}

\subsection{Asymptotic Behavior of the PMPO Objective}

We now formally demonstrate that the PMPO objective encompasses arithmetic and geometric mean-based counterparts as limiting cases, establishing its role as a unifying theoretical framework.

\begin{proposition}[Convergence to the GRPO Objective] 
As the exponent $p \to 1$, the per-sequence PMPO objective $\mathcal{J}_i(\theta)$ asymptotically converges to the standard GRPO objective.
\end{proposition}

\begin{proof}
By direct substitution of $p=1$ and $n=|o_i|$ into the power mean definition (\textit{i.e.}, Eq. (\ref{equ:power})), we get the arithmetic mean:
\begin{equation}
    \lim_{p \to 1} M_p(\Phi_{i}(\theta)) = M_1(\Phi_{i}(\theta)) = \frac{1}{|o_i|} \sum_{t=1}^{|o_i|} \phi_{i,t}(\theta)
\end{equation}

Substituting this into the full PMPO objective yields:
\begin{equation}
\begin{aligned}
&\lim_{p \to 1} \mathcal{J}_i(\theta) = \text{sgn}(\hat{A}_i) \cdot \left( \frac{1}{|o_i|} \sum_{t=1}^{|o_i|} \phi_{i,t}(\theta) \right) \\
&= \frac{1}{|o_i|} \sum_{t=1}^{|o_i|} \text{sgn}(\hat{A}_i) \cdot |\min(r_{i,t}(\theta)\hat{A}_i, \rho_{i,t}(\theta)\hat{A}_i)|
\end{aligned}
\end{equation}

In particular, $\text{sgn}(X) \cdot |X|=X$. Since the term $\min(r_{i,t}(\theta)\hat{A}_i, \rho_{i,t}(\theta)\hat{A}_i)$ inherently preserves the direction of $\hat{A}_i$, it follows that:
\begin{equation}
\begin{aligned}
    &\text{sgn}(\hat{A}_i) \cdot |\min(r_{i,t}(\theta)\hat{A}_i, \rho_{i,t}(\theta)\hat{A}_i)| \\
    &= \min(r_{i,t}(\theta)\hat{A}_i, \rho_{i,t}(\theta)\hat{A}_i)
\end{aligned}
\end{equation}

Therefore, the objective simplifies to:
\begin{equation}
    \lim_{p \to 1} \mathcal{J}_i(\theta) = \frac{1}{|o_i|} \sum_{t=1}^{|o_i|} \min(r_{i,t}(\theta)\hat{A}_i, \rho_{i,t}(\theta)\hat{A}_i)
\end{equation}

This is precisely the GRPO objective. As a result, PMPO exactly recovers GRPO as a special case.

\end{proof}

\begin{proposition}[Convergence to the GMPO-Style Objective]
In the theoretical limit as $p \to 0^+$, the per-sequence PMPO objective $\mathcal{J}_i(\theta)$ converges to a geometric mean-based formulation, recovering the strict consistency constraints characteristic of GMPO.

\end{proposition}

\begin{Proof}
First, we establish the limit of the power mean term. Let $J = \lim_{p \to 0^+} M_p(\phi_{i,t}(\theta))$. This limit presents an indeterminate form of type $1^\infty$. To resolve this, we analyze the limit of $\ln(J)$. For notational simplicity, let $n = |o_i|$ and $x_t = \phi_{i,t}(\theta)$:
\begin{equation}
\begin{aligned}
   \ln(J) &= \lim_{p \to 0^+} \ln \left[ \left( \frac{1}{n} \sum_{t=1}^{n} x_t^p \right)^{\frac{1}{p}} \right] \\
   &= \lim_{p \to 0^+} \frac{\ln \left( \frac{1}{n} \sum_{t=1}^{n} x_t^p \right)}{p} \\
   &\quad \text{\textcolor{gray}{\small{(Indeterminate form $\frac{0}{0}$, applying L'Hôpital's Rule)}}} \\
   &= \lim_{p \to 0^+} \frac{\frac{d}{dp} \left( \ln \left( \frac{1}{n} \sum_{t=1}^{n} x_t^p \right) \right)}{\frac{d}{dp}(p)} \\
   &= \lim_{p \to 0^+} \frac{\frac{1}{\frac{1}{n} \sum_{t=1}^{n} x_t^p} \cdot \left( \frac{1}{n} \sum_{t=1}^{n} x_t^p \ln(x_t) \right)}{1} \\
   &= \frac{1}{\frac{1}{n} \sum_{t=1}^{n} x_t^0} \cdot \left( \frac{1}{n} \sum_{t=1}^{n} x_t^0 \ln(x_t) \right) \\
   &= \frac{1}{1} \cdot \left( \frac{1}{n} \sum_{t=1}^{n} \ln(x_t) \right) \\
   &= \frac{1}{n} \sum_{t=1}^{n} \ln(x_t) \\
   &= \ln \left( \left( \prod_{t=1}^{n} x_t \right)^{\frac{1}{n}} \right)
\end{aligned}
\end{equation}

By exponentiating both sides, we obtain:
\begin{equation}
    J = \left( \prod_{t=1}^{n} x_t \right)^{\frac{1}{n}} = \left( \prod_{t=1}^{|o_i|} \phi_{i,t}(\theta) \right)^{\frac{1}{|o_i|}}
\end{equation}

Substituting this back into the full PMPO objective:
\begin{equation}
    \lim_{p \to 0^+} \mathcal{J}_i(\theta) = \text{sgn}(\hat{A}_i) \cdot \left( \prod_{t=1}^{|o_i|} \phi_{i,t}(\theta) \right)^{\frac{1}{|o_i|}}
\end{equation}

Therefore, the PMPO objective converges to a geometric aggregation of token-level magnitudes.
\end{Proof}

\subsection{Conclusion}

By tuning $p$ based on model's performance, PMPO fluidly interpolates between a high-sensitivity pattern and a high-stability pattern. This allows it to effectively balance the pursuit of high-reward outliers with stable policy refinement, thereby providing a more adaptive learning mechanism.

\section{Appendix E: Theoretical Justification of Gradient Dynamics in PMPO} \label{appendix:theory_dynamics}

In this section, we provide a unified theoretical framework to analyze the superiority of PMPO. We employ \textit{Gradient Sensitivity Analysis} to showcase how PMPO adaptively balances signal amplification and noise suppression.

\subsection{Gradient Sensitivity and the Crossover Point}

We analyze the policy update dynamics by examining the \textit{effective gradient contribution} of the advantage signal $A$, which is modeled as a transformation function $\Psi(A)$. 

(1) For GRPO (arithmetic mean), the objective is linear in $A$ (\textit{i.e.}, $\mathcal{J} \sim \Sigma A_i$). Therefore, the gradient contribution is directly proportional to the magnitude of the advantage itself, which is given as $\Psi_{GRPO}(A) = A$. 

(2) For PMPO (power mean), the objective involves the $p$-th power of advantages (\textit{i.e.}, $\mathcal{J} \sim (\Sigma A_i^p)^{1/p}$). Crucially, when computing the gradient with respect to the policy parameters, the chain rule introduces a factor proportional to $A^{p-1}$ scaling the original gradient direction. Therefore, the effective gradient scaling factor behaves as $\Psi_{PMPO}(A) \propto A^p$ where $p \in (0, 1]$.

\begin{definition}[Gradient Sensitivity]
The sensitivity of the policy update with respect to the advantage signal magnitude $A$ is defined as $S(A) = \frac{\partial \Psi(A)}{\partial A}$.
\end{definition}

For GRPO, $S_{GRPO}(A) = 1$ (constant sensitivity). For APMPO, $S_{PMPO}(A) = p A^{p-1}$.

\begin{proposition}[Dual-Phase Sensitivity and Unique Crossover]
For any fixed adaptive parameter $p \in (0, 1)$, there exists a unique Crossover Point $A^*$ such that PMPO exhibits two distinct behaviors relative to GRPO:
\begin{equation}
    A^* = p^{\frac{1}{1-p}}
\end{equation}

The sensitivity ratio $\rho(A) = \frac{S_{PMPO}(A)}{S_{GRPO}(A)}$ satisfies:
\begin{enumerate}
    \item \textbf{Signal Boosting Phase ($A < A^*$):} In this phase, $\rho(A) > 1$. Gradients for low‑advantage samples are amplified relative to the linear baseline, preventing gradient vanishing in rare but correct reasoning paths that may not yield maximum rewards.
    \item \textbf{Outlier Damping Phase ($A > A^*$):} In this phase, $\rho(A) < 1$. Gradients for high‑advantage outliers are attenuated. This functions as a \textit{gradient‑clipping} mechanism, preventing any single sample from destabilizing the policy update.
\end{enumerate}
\end{proposition}

\begin{proof}
We analyze the ratio $\rho(A) = p A^{p-1}$. To find the crossover point where sensitivities are equal, we set $\rho(A) = 1$:
\begin{equation}
\begin{aligned}
    p A^{p-1} = 1 &\implies A^{p-1} = \frac{1}{p} \\
    &\implies A = \left(\frac{1}{p}\right)^{\frac{1}{p-1}} = p^{\frac{1}{1-p}}
\end{aligned}
\end{equation}

Let $f(A) = p A^{p-1}$. Since $p \in (0, 1)$, the exponent $p-1 < 0$. As a result, $f(A)$ is strictly monotonically decreasing on $(0, \infty)$.
\begin{itemize}
    \item For $A < A^*$, $f(A) > f(A^*) = 1$, implying signal boosting.
    \item For $A > A^*$, $f(A) < f(A^*) = 1$, implying outlier suppression.
\end{itemize}

This establishes PMPO as a \textit{non-linear filter} that enhances weak signals while suppressing noise.
\end{proof}

\subsection{Adaptive Stability Analysis}

A critical feature of PMPO is the adaptive nature of $p$, \textit{i.e.}, $p = \exp(-\gamma \mu_R)$. We demonstrate that this mechanism automatically transitions the learning strategy from exploration to stabilization.

\begin{proposition}[Asymptotic Stability]
As the policy performance improves ($\mu_R \uparrow$), the adaptive parameter $p \to 0$. Consequently, the crossover point $A^*$ approaches zero:
\begin{equation}
    \lim_{p \to 0^+} A^* = 0
\end{equation}
\end{proposition}

\begin{proof}
Let $L = \lim_{p \to 0^+} \ln(A^*) = \lim_{p \to 0^+} \frac{\ln p}{1-p}$.
As $p \to 0^+$, $\ln p \to -\infty$ and $1-p \to 1$.
As a result, $L \to -\infty$.
Since $\ln(A^*) \to -\infty$, it follows that $A^* \to 0$.
\end{proof}

\begin{corollary}[Phased Learning Dynamics]
This proposition implies a theoretical guarantee for the learning schedule:
\begin{itemize}
    \item \textbf{Early Training (Low $\mu_R \implies p \approx 1$):} This indicates minimal non‑linear distortion. In this regime, PMPO operates almost linearly (similar to GRPO), allowing high‑reward outliers to exert full influence without suppression. This is crucial for signal discovery, enabling the model to rapidly capture the earliest correct reasoning paths.

    \item \textbf{Late Training (High $\mu_R \implies p \approx 0$):} The crossover point $A^* \to 0$. The Outlier Damping Phase then dominates the entire signal space ($A > A^*$ for almost all $A$). In this regime, PMPO applies strong attenuation to high‑advantage signals, effectively driving the policy to optimize for consistency rather than pursuing isolated high rewards.
\end{itemize}
\end{corollary}

\section{Appendix F: Theoretical Analysis on FAC} \label{appendix:theo2}

This section provides a theoretical analysis demonstrating the superiority of FAC over static clipping mechanisms. The core argument posits that by adaptively adjusting the clipping bound based on fluctuated reward distributions, FAC circumvents the inherent trade-offs of fixed constraints.

\subsection{Preliminaries}
\begin{definition}[Static and Adaptive Clipping]
A static clipping mechanism constrains the importance sampling ratio $r_{i,t}(\theta)$ to a fixed interval $[1-\epsilon_{\text{static1}}, 1+\epsilon_{\text{static2}}]$, imposing a uniform constraint across all updates. In contrast, FAC defines an adaptive interval $[1-\epsilon_{\text{low}}, 1+\epsilon_{\text{ada}}]$, where the upper bound $\epsilon_{\text{ada}}$ is adaptively adjusted.
\end{definition}

\begin{definition}[Feedback Stability Score]
We utilize the FSS from Eq. (\ref{eq:snr}) as a quantitative proxy for signal stability within a batch.
\end{definition}

We make the following assumption:
\begin{assumption}[FSS as a Quality Proxy]
FSS is positively correlated with the stability of the positive learning signal. High FSS implies high-fidelity reinforcement signals, whereas low FSS indicates unstable signals necessitating caution.
\end{assumption}

\subsection{Design Rationale for Asymmetric Clipping}
A pivotal design feature in FAC is its asymmetric clipping mechanism, characterized by a \textit{fixed lower bound} $1 - \epsilon_{\text{low}}$ and an \textit{adaptive upper bound} $1 + \epsilon_{\text{ada}}$. This architecture stems from the distinct roles of positive and negative feedback in reasoning tasks.

\paragraph{Adaptive Upper Bound for Conservative Reinforcement.} 
The upper bound regulates positive reinforcement updates (\textit{i.e.}, $\hat{A} > 0$). The primary risk in this context involves overfitting to spurious correlations common in reward-sparse reasoning tasks. The stability of positive signals is contingent upon batch statistics. Concretely, a high FSS indicates high-fidelity feedback, justifying a larger update to accelerate convergence. Conversely, a low FSS implies that positive signals may represent outliers within a noisy distribution. In such cases, tightening the upper bound mitigates the risk of the model confidently committing to spurious success. Consequently, modulating $\epsilon_{\text{ada}}$ via FSS enforces \textit{risk aversion} under low-quality signals.

\paragraph{Fixed Lower Bound for Decisive Pruning.} 
In contrast, the lower bound governs negative penalty updates (\textit{i.e.}, $\hat{A} < 0$). A fixed $\epsilon_{\text{low}}$ is employed based on the premise that \textit{negative signals are inherently more stable than positive ones}. Even in low FSS batches, reasoning paths yielding incorrect answers constitute definitive signals warranting penalization. Adaptively tightening the lower bound alongside the upper bound during low-FSS phases would effectively freeze the policy, thereby hindering error correction. By maintaining a fixed lower bound, a consistent mechanism for policy adjustment is preserved, enabling the decisive pruning of incorrect reasoning paths even when confidence for positive reinforcement is insufficient. This asymmetry establishes a \textit{``Conservative Reinforcement, Decisive Pruning''} dynamic essential for improved reasoning.

\subsection{Analysis of the Adaptive Mechanism}
The efficacy of FAC stems from its ability to map signal quality to an optimal trust region.

\begin{lemma}[Properties of the FAC Mapping Function]
The mapping from FSS to the adaptive clipping bound $\epsilon_{\text{ada}}$ is bounded, monotonic, and smooth.
\end{lemma}

\begin{proof}
We establish these properties for the mapping function:

(1) \textbf{Boundedness}: The $\tanh(\cdot)$ function maps $\text{FSS} \in (0, \infty)$ to $(0,1)$. This constrains $\epsilon_{\text{ada}}$ to the strictly positive interval $(\epsilon_{\text{min}}, \epsilon_{\text{max}})$, preventing vanishing or exploding trust regions.

(2) \textbf{Monotonicity}: Since $\tanh(\cdot)$ is strictly increasing, $\epsilon_{\text{ada}}$ scales positively with FSS. This aligns with our design rationale that higher FSS warrants a larger trust region.

(3) \textbf{Smoothness}: The differentiability of $\tanh(\cdot)$ ensures smooth transitions in $\epsilon_{\text{ada}}$, thereby preventing abrupt shifts in learning dynamics.
\end{proof}

\begin{proposition}[FAC Overcomes the Static Clipping Dilemma] \label{theo:FAC}
Static clipping strategies face an intrinsic dilemma, where a small $\epsilon_{\text{static}}$ ensures stability but retards learning on good batches, while a large $\epsilon_{\text{static}}$ accelerates learning but risks instability on noisy batches. FAC resolves this by adaptively tuning $\epsilon_{\text{ada}}$ based on FSS.
\end{proposition}

\begin{proof}
We analyze the behavior of FAC in distinct scenarios:

(1) \textbf{Case 1: Stable Reward Signal ($\text{FSS} \gg 0$)}: The batch exhibits high-reward trajectories. As $\text{FSS} \to \infty$, $\tanh(\text{FSS}) \to 1$, and $\epsilon_{\text{ada}} \to \epsilon_{\text{max}}$. FAC automatically expands the trust region, enabling aggressive exploitation of high-quality signals. A static method would unnecessarily restrict this valid update and slow down convergence.

(2) \textbf{Case 2: Unstable Reward Signal ($\text{FSS} \approx 0$)}: 
The batch exhibits pronounced reward fluctuations, suggesting that positive advantage signals are likely unreliable. As $\text{FSS} \to 0$, $\tanh(\text{FSS}) \to 0$, and $\epsilon_{\text{ada}} \to (\epsilon_{\text{min}}+\epsilon_{\text{max}})/2$. FAC tightens the trust region to limit updates driven by unreliable samples. This mitigates overfitting to unstable learning signals, whereas a fixed static bound would risk policy degradation.
\end{proof}

\begin{remark}[Transition in Intermediate FSS Regime]
Beyond the above two extremes, the behavior of FAC in the intermediate regime (\textit{i.e.}, $0 < \text{FSS} < 1$) offers a critical advantage. 
Recall that $\tanh(x) \approx x$ for small $x$, the adaptive threshold $\epsilon_{ada}$ scales approximately linearly with FSS given as:
\begin{equation}
    \epsilon_{ada} \approx \epsilon_{min} + (\epsilon_{max} - \epsilon_{min}) \cdot \text{FSS}
\end{equation}

As FSS improves, FAC linearly relaxes the trust region. This allows for fine-grained modulation of the update step size, effectively stabilizing training during the critical phase where the model transitions from exploration to consistent reasoning.
\end{remark}

\section{Appendix G: Adaptive Trust Region and Monotonic Improvement} \label{appendix:trust}

This section provides the theoretical underpinning for FAC. We demonstrate that maximizing the lower bound of monotonic policy improvement necessitates adaptively scaling the trust region size inversely with reward volatility.

\subsection{Summary of Notations}
For clarity, Table \ref{tab:notation_theory} summarizes the key mathematical notations utilized in this analysis.

\begin{table*}[h]
    \centering
    \small
    \renewcommand{\arraystretch}{1.3}
    \begin{tabular}{l|l}
    \toprule
    \textbf{Notation} & \textbf{Description} \\
    \midrule
    $\pi_{old}, \pi$ & The policy before and after the update step. \\
    $J(\pi)$ & The \textbf{true expected return} of policy $\pi$ in the environment. \\
    $L_{true}(\pi)$ & The \textbf{theoretical surrogate objective} using exact advantage values. \\
    $\hat{L}(\pi)$ & The \textbf{empirical surrogate objective} using sampled advantages. \\
    $\mathcal{J}(\theta)$ & Shorthand for the objective function parameterized by $\theta$. \\
    $\sigma_R$ & The standard deviation of the reward estimation noise. \\
    $D_{KL}(\cdot \parallel \cdot)$ & KL divergence measuring the distance between policies. \\
    $\epsilon$ & The clipping bound hyperparameter used in algorithms such as GRPO. \\
    $\beta$ & Coefficient reflecting the value function curvature (\textit{i.e.}, policy interpolation cost). \\
    $\alpha$ & Scaling factor for the estimation noise. \\
    $\gamma$ & Discount factor in RL. \\
    \bottomrule
    \end{tabular}
    \caption{Summary of notations used in the theoretical derivation.}
    \label{tab:notation_theory}
\end{table*}

\subsection{Theoretical Analysis}

\begin{lemma}[Error Bound under Reward Uncertainty]
\label{lemma:error_bound}
Let $J(\pi)$ denote the true expected return and $\hat{L}(\pi)$ be the empirical surrogate objective constructed from sampled rewards. Assuming the reward estimation noise is bounded by its standard deviation $\sigma_R$, the approximation error is bounded by:
\begin{equation}
    \left| J(\pi) - \hat{L}(\pi) \right| \leq \beta \cdot D_{KL}(\pi \parallel \pi_{old}) + \alpha \cdot \sigma_R
\end{equation}
where $\beta > 0$ is the Lipschitz constant related to the maximum advantage, and $\alpha > 0$ scales with the importance sampling weights.
\end{lemma}

\begin{proof}
We decompose the total error into the \textit{policy interpolation error} (due to distribution shift) and the \textit{estimation error} (due to reward noise).
\begin{equation}
\begin{aligned}
  |J(\pi) - \hat{L}(\pi)| &\leq \underbrace{|J(\pi) - L_{true}(\pi)|}_{\text{(I) Interpolation Error}} \\ &+ \underbrace{|L_{true}(\pi) - \hat{L}(\pi)|}_{\text{(II) Estimation Error}}  
\end{aligned}
\end{equation}

\textbf{Part I: Bounding the Policy Interpolation Error.}
The difference between the true objective and the theoretical surrogate is bounded by the quadratic variation of the policy \cite{schulman2015trust}.
Utilizing Pinsker's inequality to relate Total Variation divergence to KL divergence, we have:
\begin{equation}
   |J(\pi) - L_{true}(\pi)| \leq C \cdot \max_s D_{KL}(\pi_{old}(\cdot|s) \parallel \pi(\cdot|s)) 
\end{equation}

Setting $\beta = C$, this yields the first term $\beta \cdot D_{KL}(\pi \parallel \pi_{old})$. (A detailed derivation of this bound is provided in the subsequent subsection).

\textbf{Part II: Bounding the Estimation Error.}
The empirical surrogate $\hat{L}$ deviates from $L_{true}$ due to advantage estimation noise. Let $\hat{A}(s,a) = A_{\pi_{old}}(s,a) + \xi(s,a)$, where $\xi$ is noise with variance $\sigma_R^2$.
\begin{equation}
    \begin{aligned}
        |L_{true}(\pi) - \hat{L}(\pi)| &= \left| \mathbb{E}_{\tau} \left[ \frac{\pi(a|s)}{\pi_{old}(a|s)} (A_{\pi_{old}} - \hat{A}) \right] \right| \\
        &= \left| \mathbb{E}_{\tau} \left[ r_t(\theta) \cdot \xi_t \right] \right|
    \end{aligned}
\end{equation}

Applying the Cauchy-Schwarz inequality $|\mathbb{E}[XY]| \leq \sqrt{\mathbb{E}[X^2]\mathbb{E}[Y^2]}$ and assuming $r_t(\theta)$ are well-behaved within the trust region, the error is dominated by the noise magnitude:
\begin{equation}
    |L_{true}(\pi) - \hat{L}(\pi)| \leq \underbrace{\sqrt{\mathbb{E}[r_t^2]}}_{\alpha} \cdot \underbrace{\sqrt{\mathbb{E}[\xi_t^2]}}_{\sigma_R} = \alpha \cdot \sigma_R
\end{equation}

Combining Parts I and II completes the proof.
\end{proof}

\subsection{Detailed Derivation of the Interpolation Bound}
\textit{(Note: This subsection expands on Part I of Lemma \ref{lemma:error_bound}.)}

\begin{proof}
The derivation unfolds in three steps:

\paragraph{Step 1: Performance Difference Lemma.}
The performance difference between two policies is given by:
\begin{equation}
    J(\pi) - J(\pi_{old}) = \mathbb{E}_{\tau \sim \pi}\left[ \sum_{t=0}^\infty \gamma^t A_{\pi_{old}}(s_t, a_t) \right]
\end{equation}
where $\gamma$ is the discount factor. The surrogate $L_{true}(\pi)$ approximates this by sampling trajectories from $\pi_{old}$:
\begin{equation}
    L_{true}(\pi) = J(\pi_{old}) + \mathbb{E}_{\tau \sim \pi_{old}} \left[ \sum_{t=0}^\infty \gamma^t \rho_t A_{\pi_{old}}(s_t, a_t) \right]
\end{equation}

The discrepancy $|J(\pi) - L_{true}(\pi)|$ arises solely from the distributional shift in different distributions.

\paragraph{Step 2: Bounding via Total Variation.}
Let $\epsilon_{tv} = \max_s D_{TV}(\pi \parallel \pi_{old})$. The $L_1$ distance between state distributions is bounded by $2(1-(1-\epsilon_{tv})^t)$.
Given a maximum advantage $\epsilon_{adv}$, the accumulated error is:
\begin{equation}
    |J(\pi) - L_{true}(\pi)| \leq \frac{2\gamma \epsilon_{adv}}{(1-\gamma)^2} \cdot \epsilon_{tv}
\end{equation}

\paragraph{Step 3: Relating to KL Divergence.}
Using the Pinsker Inequality $D_{TV} \leq \sqrt{D_{KL}/2}$ and standard quadratic approximations, we define the penalty coefficient $\beta = \frac{2\gamma \epsilon_{adv}}{(1-\gamma)^2}$. This yields:
\begin{equation}
    |J(\pi) - L_{true}(\pi)| \leq \beta \cdot D_{KL}(\pi \parallel \pi_{old})
\end{equation}
\end{proof}

\begin{theorem}[Optimality of Adaptive Clipping]
To maximize the lower bound of monotonic policy improvement (\textit{i.e.}, $J(\pi_{new}) \ge J(\pi_{old})$), the clipping bound $\epsilon$ must be adaptively scaled inversely with reward variance:
\begin{equation}
    \epsilon \propto \frac{1}{f(\sigma_R)} \approx \text{FSS}
\end{equation}
\end{theorem}

\begin{proof}
Combining Lemma \ref{lemma:error_bound} with the policy improvement identity, the lower bound of the true improvement is:
\begin{equation}
\label{eq:lower_bound}
  \Delta J(\pi) \ge \underbrace{\Delta L(\pi)}_{\text{Surrogate Gain}} - \underbrace{\left( \beta \cdot D_{KL}(\pi \parallel \pi_{old}) + \alpha \cdot \sigma_R \right)}_{\text{Total Penalty}}
\end{equation}

In GRPO, the clipping parameter $\epsilon$ serves as a hard constraint on policy divergence, specifically $D_{KL} \approx \mathcal{O}(\epsilon^2)$.
To guarantee monotonic improvement ($\Delta J(\pi) > 0$), the Surrogate Gain must outweigh the Total Penalty.

We analyze the impact of increasing $\sigma_R$:
\begin{itemize}
    \item As $\sigma_R$ increases, the penalty $\alpha \cdot \sigma_R$ grows, consuming a larger portion of the potential gain.
    \item To maintain $\Delta J > 0$, the remaining penalty term $\beta \cdot D_{KL}$ must be minimized.
    \item Since $D_{KL} \propto \epsilon^2$, this necessitates reducing $\epsilon$.
\end{itemize}

Mathematically, the optimal trust region $\epsilon^*$ that balancing these terms satisfies $\epsilon^* \propto 1 / \sigma_R$.
FAC operationalizes this via $\epsilon_{ada} \propto \text{FSS} \approx \frac{\mu}{\sigma_R}$. This mechanism tightens the clipping bound when $\sigma_R$ is high and relaxes it when $\sigma_R$ is low.
\end{proof}

\section{Appendix H: Gradient Derivation of the APMPO Objective} \label{appendix:gradient}

In this section, we provide a detailed derivation for the gradient of the APMPO objective function $\mathcal{J}(\theta)$ with respect to the model parameters $\theta$. The overall objective is the average of per-sample objectives:

\begin{equation}
    \mathcal{J}(\theta) = \frac{1}{G} \sum_{i=1}^{G} \mathcal{J}_i(\theta)
\end{equation}

The gradient is computed as:
\begin{equation}
    \nabla_\theta \mathcal{J}(\theta) = \frac{1}{G} \sum_{i=1}^{G} \nabla_\theta \mathcal{J}_i(\theta)
\end{equation}

\subsection{Per-Sample Objective Function}

The per-sample objective $L_i(\theta)$ is defined as:
\begin{equation}
    \mathcal{J}_i(\theta) = \left[\frac{1}{|o_i|} \sum_{t=1}^{|o_i|} \left( \phi_{i,t}(\theta) \right)^p \right]^{\frac{1}{p}} \cdot \text{sgn}(\hat{A}_i)
    \label{eq:appendix_Li}
\end{equation}
where $\phi_{i,t}(\theta)$ is the token-level magnitude:
\begin{equation}
    \phi_{i,t}(\theta) = \left| \min\left( r_{i,t}(\theta)\hat{A}_i, \rho_{i,t}(\theta)\hat{A}_i \right) \right|
    \label{eq:appendix_phi}
\end{equation}

\subsection{Applying the Chain Rule}
We apply the chain rule to compute $\nabla_\theta \mathcal{J}_i(\theta)$. For simplicity, let $M_p(\theta) = \left[ \frac{1}{|o_i|} \sum_{t} (\phi_{i,t}(\theta))^p \right]^{\frac{1}{p}}$. The objective is $\mathcal{J}_i(\theta) = M_p(\theta) \cdot \text{sgn}(\hat{A}_i)$. Since $\text{sgn}(\hat{A}_i)$ is a constant scalar, the gradient is:
\begin{equation}
    \nabla_\theta \mathcal{J}_i(\theta) = \text{sgn}(\hat{A}_i) \cdot \nabla_\theta M_p(\theta)
    \label{eq:appendix_grad_L}
\end{equation}

Now we compute the gradient of the power-mean term $M_p(\theta)$. Let $S(\theta) = \frac{1}{|o_i|} \sum_{t} (\phi_{i,t}(\theta))^p$, so $M_p(\theta) = (S(\theta))^{\frac{1}{p}}$. We can derive that:
\begin{equation}
\begin{aligned}
    &\nabla_\theta M_p(\theta) = \frac{1}{p} (S(\theta))^{\frac{1}{p} - 1} \cdot \nabla_\theta S(\theta) \\
    &= \frac{1}{p} (S(\theta))^{\frac{1-p}{p}} \cdot \frac{1}{|o_i|} \sum_{t} \Big[ p (\phi_{i,t}(\theta))^{p-1} \cdot \nabla_\theta \phi_{i,t}(\theta) \Big] \\
    &=(M_p(\theta))^{1-p} \cdot \frac{1}{|o_i|} \sum_{t} \Big[(\phi_{i,t}(\theta))^{p-1} \cdot \nabla_\theta \phi_{i,t}(\theta) \Big]
    \label{eq:appendix_grad_Mp}
\end{aligned}
\end{equation}

\subsection{Gradient of \texorpdfstring{$\phi_{i,t}(\theta)$}{phi}}
Let $U_{i,t}(\theta) = \min(r_{i,t}(\theta)\hat{A}_i, \rho_{i,t}(\theta)\hat{A}_i)$, we can know that $\phi_{i,t}(\theta) = |U_{i,t}(\theta)|$. Using the chain rule and the subgradient of the absolute value function ($\frac{d|x|}{dx} = \text{sgn}(x)$), we get:
\begin{equation}
    \nabla_\theta \phi_{i,t}(\theta) = \text{sgn}(U_{i,t}(\theta)) \cdot \nabla_\theta U_{i,t}(\theta)
    \label{eq:appendix_grad_phi_chain}
\end{equation}

The term $U_{i,t}(\theta)$ is a minimum of two terms. Both $r_{i,t}(\theta)$ and $\rho_{i,t}(\theta)$ are positive ratios close to 1. Therefore, the sign of $U_{i,t}(\theta)$ is determined solely by the sign of $\hat{A}_i$. Therefore, we can know that $\text{sgn}(U_{i,t}(\theta)) = \text{sgn}(\hat{A}_i)$.

Next, the subgradient of the $\min$ function is given by:
\begin{equation}
\nabla_\theta U_{i,t}(\theta) = \begin{cases}
    \hat{A}_i \cdot \nabla_\theta r_{i,t}(\theta) & \text{if } r_{i,t}(\theta) < \rho_{i,t}(\theta) \\
    \hat{A}_i \cdot \nabla_\theta \rho_{i,t}(\theta) & \text{if } \rho_{i,t}(\theta) < r_{i,t}(\theta)
\end{cases}
\label{eq:u(theta)}
\end{equation}

Substituting back into Eq.~(\ref{eq:appendix_grad_phi_chain}):
\begin{align}
    &\nabla_\theta \phi_{i,t}(\theta) = \text{sgn}(\hat{A}_i) \cdot \nabla_\theta U_{i,t}(\theta) \\
    &= \begin{cases}
    \text{sgn}(\hat{A}_i) \cdot \hat{A}_i \cdot \nabla_\theta r_{i,t}(\theta) & \text{if } r_{i,t}(\theta) < \rho_{i,t}(\theta) \\
    \text{sgn}(\hat{A}_i) \cdot \hat{A}_i \cdot \nabla_\theta \rho_{i,t}(\theta) & \text{if } \rho_{i,t}(\theta) < r_{i,t}(\theta)
\end{cases}
    \label{eq:appendix_grad_phi_final}
\end{align}

\subsection{Assembling the Final Gradient}
Now we substitute the gradient of $M_p$ (Eq.~(\ref{eq:appendix_grad_Mp})) back into the main gradient expression (Eq.~(\ref{eq:appendix_grad_L})).
\begin{equation}
\begin{aligned}
    \nabla_\theta \mathcal{J}_i(\theta) &= 
    \text{sgn}(\hat{A}_i) \cdot (M_p(\theta))^{1-p} \\
    &\cdot \frac{1}{|o_i|} \sum_{t=1}^{|o_i|} \left[ (\phi_{i,t}(\theta))^{p-1} \cdot \nabla_\theta \phi_{i,t}(\theta) \right]
\end{aligned}
\end{equation}

Next, we substitute the subgradient of $\phi_{i,t}(\theta)$ from Eq.~(\ref{eq:appendix_grad_phi_final}) into the expression above. This yields:
\begin{equation}
\begin{aligned}
    \nabla_\theta \mathcal{J}_i(\theta) &= 
    \text{sgn}(\hat{A}_i) \cdot (M_p(\theta))^{1-p} \cdot \frac{1}{|o_i|} \\
    &\cdot \sum_{t=1}^{|o_i|} \Big[ (\phi_{i,t}(\theta))^{p-1} \cdot \underbrace{\text{sgn}(\hat{A}_i) \cdot \nabla_\theta U_{i,t}(\theta)}_{\nabla_\theta \phi_{i,t}(\theta)} \Big]
\end{aligned}
\end{equation}

Since $\text{sgn}(\hat{A}_i) \cdot \text{sgn}(\hat{A}_i) = (\text{sgn}(\hat{A}_i))^2 = 1$, this simplifies the expression to:
\begin{equation} \label{eq:appendix_grad_simplified}
\begin{aligned}
    \nabla_\theta \mathcal{J}_i(\theta) &= (M_p(\theta))^{1-p} \cdot \frac{1}{|o_i|} \sum_{t=1}^{|o_i|} \Big[ (\phi_{i,t}(\theta))^{p-1} \\
    &\cdot \nabla_\theta U_{i,t}(\theta) \Big]
\end{aligned}
\end{equation}
where $\nabla_\theta U_{i,t}(\theta)$ is given in Eq. (\ref{eq:u(theta)}).

We define the token-level weight $w_{i,t}(\theta)$ as:
\begin{equation}
    w_{i,t}(\theta) = (M_p(\theta))^{1-p} \cdot (\phi_{i,t}(\theta))^{p-1}
    \label{eq:w(theta)}
\end{equation}

We can derive that:
\begin{equation}
    \nabla_\theta \mathcal{J}_i(\theta) = \frac{1}{|o_i|} \sum_{t=1}^{|o_i|} \left[w_{i,t}\cdot \nabla_\theta U_{i,t}(\theta) \right]
\end{equation}

Assuming for simplicity that $r_{i,t} < \rho_{i,t}$, The final gradient expression becomes:
\begin{equation}
    \begin{split}
    \nabla_\theta L_i(\theta) &\approx \frac{1}{|o_i|} \sum_{t} \left[ w_{i,t}(\theta) \cdot \hat{A}_i \cdot \nabla_\theta r_{i,t}(\theta) \right] \\
    &\approx \frac{1}{|o_i|} \sum_{t} \Big[ w_{i,t}(\theta) \cdot \hat{A}_i \cdot r_{i,t}(\theta) \\
    &\cdot \nabla_\theta \log \pi_\theta(o_{i,t}|o_{i,<t}) \Big]
    \label{eq:appendix_grad_final_simple}
\end{split}
\end{equation}
where the last step uses the log-derivative technique $\nabla_\theta r_{i,t}(\theta) = r_{i,t}(\theta) \cdot \nabla_\theta \log \pi_\theta(o_{i,t}|o_{i,<t})$.

\subsection{Analysis}
The final gradient form in Eq.~(\ref{eq:appendix_grad_final_simple}) can be interpreted as follows.

\paragraph{Direction.} The gradient direction is determined by $\hat{A}_i$ and $\nabla_\theta \log \pi_\theta$.

\paragraph{Magnitude Modulation.} The core innovation of APMPO lies in the adaptive weight $w_{i,t}(\theta)$, particularly its token-level component $(\phi_{i,t}(\theta))^{p-1}$.
\begin{itemize}
    \item \textbf{When $p \to 1$ (GRPO-like):} We can know that $(\phi_{i,t}(\theta))^{p-1} \to 1$. The weights $w_{i,t}(\theta)$ become approximately uniform across tokens, and the gradient resembles the standard GRPO gradient.
    \item \textbf{When $p \to 0$ (GMPO-like):} We can know that $(\phi_{i,t}(\theta))^{p-1} = 1/\phi_{i,t}(\theta)$. This means that tokens with a larger update magnitude $\phi$ receive smaller weight in the gradient summation. This leads to a smoother gradient update.
\end{itemize}

This derivation shows that the APMPO objective results in a well-behaved gradient. Its core mechanism is to adaptively re-weight the contribution of each token to the total gradient, thereby smoothly interpolating between an aggressive and a conservative update strategy based on the exponent $p$.

\section{Appendix I: Convergence Analysis of APMPO} \label{appendix:convergence}
\newcommand{\E}{\mathbb{E}}
\newcommand{\norm}[1]{\left\|#1\right\|}
\newcommand{\abs}[1]{\left|#1\right|}
\newcommand{\inner}[2]{\left\langle #1, #2 \right\rangle}
\newcommand{\TrueObj}{\mathcal{F}}

In this section, we provide a theoretical analysis of the convergence properties of APMPO. To address the challenge that the power mean $p$ and clipping threshold $\epsilon_\text{ada}$) evolve adaptively with the policy, we treat the adaptive hyperparameters as functions of the policy parameters $\theta$ and analyze the convergence of the implicit composite objective.

\subsection{Problem Setup}
Let $\eta$ represent the vector of adaptive hyperparameters $\eta = [p, \epsilon_{ada}]^\top$. In APMPO, these parameters are determined by the reward statistics of the current policy $\pi_\theta$. We explicitly denote this dependency as $\eta(\theta)$.

APMPO optimizes a parameterized objective $f(\theta; \eta)$. However, since $\eta$ is updated continuously, the \textit{true} implicit objective we aim to maximize is:
\begin{equation}
    \TrueObj(\theta) \coloneqq f(\theta; \eta(\theta))
\end{equation}

The parameter update follows stochastic gradient descent (SGD) using a stochastic estimator $g_k$. Crucially, standard RL compute the gradient of $f$ assuming $\eta$ is fixed, ignoring the path dependence of $\eta$ on $\theta$. The update rule is:
\begin{equation}
    \theta_{k+1} = \theta_k - \alpha_k g_k
\end{equation}
where $g_k$ estimates the partial gradient $\nabla_\theta f(\theta; \eta)|_{\eta=\eta(\theta_k)}$.

\subsection{Assumptions and Plausibility Analysis}

\begin{assumption}[Smoothness of the Parametric Objective] \label{ass:smoothness}
For any fixed $\eta$, the function $f(\cdot; \eta)$ is $L_f$-smooth. Furthermore, the gradient is bounded by the clipping mechanism:
\begin{equation}
    \norm{\nabla_\theta f(\theta; \eta)} \le M_f
\end{equation}
\end{assumption}

\begin{remark}[\textbf{Plausibility}]
This is structurally enforced by APMPO's design. FAC constrains the sampling ratio $r_{i,t}(\theta)$ and limits the gradient magnitude via $\epsilon$-clipping. This explicit constraint ensures the objective behaves as a locally Lipschitz function.
\end{remark}

\begin{assumption}[Stochastic Gradient Oracle] \label{ass:subgradient}
The stochastic gradient $g_k$ is an unbiased estimator of the \textit{partial} gradient with bounded variance:
\begin{enumerate}
    \item $\E[g_k | \theta_k] = \nabla_\theta f(\theta_k; \eta(\theta_k))$.
    \item $\E[\norm{g_k - \nabla_\theta f(\theta_k; \eta(\theta_k))}^2] \le \sigma^2$.
\end{enumerate}
\end{assumption}

\begin{assumption}[Lipschitz Continuity of Adaptive Parameters] \label{ass:lipschitz_params}
The mapping from policy parameters to adaptive hyperparameters, $\theta \mapsto \eta(\theta)$, is $L_\eta$-Lipschitz continuous. Additionally, the objective $f(\theta; \eta)$ is smooth with respect to $\eta$, with bounded gradient $\norm{\nabla_\eta f} \le M_\eta$.
\end{assumption}

\begin{remark}[\textbf{Plausibility: Smoothness of Expectation}]
The parameters $p$ and $\epsilon_{ada}$ are derived from reward statistics $\mu_R(\theta) = \E_{\tau \sim \pi_\theta}[R(\tau)]$ and $\sigma_R(\theta)$. Since the policy $\pi_\theta$  is a smooth function of $\theta$ and rewards are bounded, the expectation $\mu_R(\theta)$ is Lipschitz continuous w.r.t. $\theta$. Since $p(\mu_R)$ (exponential) and $\epsilon_{ada}(\sigma_R)$ (tanh) are smooth bounded functions, the composite mapping $\eta(\theta)$ preserves Lipschitz continuity.
\end{remark}

\subsection{Convergence Proof}

We analyze the convergence of the true implicit objective $\TrueObj(\theta)$. The key challenge is that our update direction $g_k$ approximates the \textit{partial} gradient $\nabla_\theta f$, whereas the true gradient of $\TrueObj(\theta)$ includes a total derivative term.

By the Chain Rule, the true gradient is:
\begin{equation} \label{eq:total_grad}
    \nabla \TrueObj(\theta) = \nabla_\theta f(\theta; \eta) + \nabla_\eta f(\theta; \eta)^\top \nabla_\theta \eta(\theta)
\end{equation}

Let $B(\theta) \coloneqq \nabla_\eta f(\theta; \eta)^\top \nabla_\theta \eta(\theta)$ denote the approximation bias. Under Assumption \ref{ass:lipschitz_params}, this bias is bounded:
\begin{equation}
    \norm{B(\theta)} \le \norm{\nabla_\eta f} \norm{\nabla_\theta \eta} \le M_\eta L_\eta \coloneqq C_{bias}
\end{equation}

Therefore, our algorithm performs SGD with a \textit{biased} gradient estimator.

\begin{theorem}[Convergence with Bounded Bias]
Under Assumptions \ref{ass:smoothness}--\ref{ass:lipschitz_params}, with a learning rate schedule satisfying $\sum \alpha_k = \infty$ and $\sum \alpha_k^2 < \infty$, the algorithm converges to a neighborhood of a stationary point of $\TrueObj(\theta)$. Specifically:
\begin{equation}
    \liminf_{k\to\infty} \E\left[ \norm{\nabla \TrueObj(\theta_k)}^2 \right] \le K \cdot C_{bias}^2
\end{equation}
for some constant $K$, implying convergence up to the limit of the adaptive drift.
\end{theorem}

\begin{proof}
Since $\TrueObj$ is a composition of smooth functions, we assume it is $L$-smooth. From the Descent Lemma:
\begin{equation}
\begin{aligned}
   \TrueObj(\theta_{k+1}) &\le \TrueObj(\theta_k) + \inner{\nabla \TrueObj(\theta_k)}{\theta_{k+1} - \theta_k} \\
   &+ \frac{L}{2}\norm{\theta_{k+1} - \theta_k}^2 
\end{aligned}
\end{equation}

Substituting the update $\theta_{k+1} - \theta_k = -\alpha_k g_k$:
\begin{equation}
    \TrueObj(\theta_{k+1}) \le \TrueObj(\theta_k) - \alpha_k \inner{\nabla \TrueObj(\theta_k)}{g_k} + \frac{L \alpha_k^2}{2} \norm{g_k}^2
\end{equation}

Taking expectations conditioned on $\theta_k$. Note that $\E[g_k] = \nabla_\theta f = \nabla \TrueObj(\theta_k) - B(\theta_k)$ (from Eq. (\ref{eq:total_grad})).
\begin{equation}
\begin{aligned}
    \E[\TrueObj(\theta_{k+1})] &\le \TrueObj(\theta_k) - \alpha_k \inner{\nabla \TrueObj(\theta_k)}{\nabla \TrueObj(\theta_k) - B(\theta_k)} \\
    &+ \frac{L \alpha_k^2}{2} \E[\norm{g_k}^2] \\
    &= \TrueObj(\theta_k) - \alpha_k \norm{\nabla \TrueObj(\theta_k)}^2 \\&+ \alpha_k \inner{\nabla \TrueObj(\theta_k)}{B(\theta_k)} + \frac{L \alpha_k^2}{2} M^2
\end{aligned}
\end{equation}

Using Young's Inequality $\inner{a}{b} \le \frac{1}{2}\norm{a}^2 + \frac{1}{2}\norm{b}^2$:
\begin{equation}
\begin{aligned}
    \E[\TrueObj(\theta_{k+1})] &\le \TrueObj(\theta_k) - \alpha_k \norm{\nabla \TrueObj(\theta_k)}^2 \\& +\frac{\alpha_k}{2}\norm{\nabla \TrueObj(\theta_k)}^2 + \frac{\alpha_k}{2}\norm{B(\theta_k)}^2 \\&+ \mathcal{O}(\alpha_k^2) \\
    &= \TrueObj(\theta_k) - \frac{\alpha_k}{2} \norm{\nabla \TrueObj(\theta_k)}^2 + \frac{\alpha_k}{2} C_{bias}^2 \\
    &+ \mathcal{O}(\alpha_k^2)
\end{aligned}
\end{equation}

Rearranging and summing over $k=0$ to $N$:
\begin{equation}
\begin{aligned}
    \sum_{k=0}^N \frac{\alpha_k}{2} \E[\norm{\nabla \TrueObj(\theta_k)}^2] &\le \TrueObj(\theta_0) - \E[\TrueObj(\theta_{N+1})] \\
    &+ \sum_{k=0}^N \frac{\alpha_k}{2} C_{bias}^2 + \sum_{k=0}^N \mathcal{O}(\alpha_k^2)
\end{aligned}
\end{equation}
s
As $N \to \infty$, for the LHS to remain consistent with the RHS (bounded objective), the gradient norm cannot remain arbitrarily large. The algorithm drives the gradient norm down until it is dominated by the bias term $C_{bias}$.

Practically, as the policy converges, the reward distribution statistics stabilize, meaning $\nabla_\theta \eta(\theta) \to 0$. Consequently, the bias $C_{bias} \to 0$, recovering standard convergence to a stationary point.
\end{proof}

\begin{remark}[Quantitative Bound on Gradient Bias]
While the qualitative boundedness of the bias term $B(\theta)$ suffices for the convergence proof, we can explicitly quantify this bound to analyze the convergence rate. 
Based on Assumption \ref{ass:lipschitz_params}, we have $\|\nabla_\theta \eta\| \leq L_\eta$. 
Furthermore, assuming the objective function's sensitivity to $\eta$ is bounded such that $\|\nabla_\eta J(\theta, \eta)\| \leq M_\eta$, applying the Cauchy-Schwarz inequality yields:
\begin{equation}
   \|B(\theta)\| = \|\nabla_\theta \eta \cdot \nabla_\eta J\| \leq \|\nabla_\theta \eta\| \cdot \|\nabla_\eta J\| \leq L_\eta \cdot M_\eta 
\end{equation}

This quantitative bound explicitly connects the approximation error to the stability of the adaptive parameter. It implies that as the policy stabilizes (\textit{i.e.}, $\nabla_\theta \eta \to 0$ and $L_\eta \to 0$ locally), the bias term vanishes asymptotically, ensuring that APMPO recovers the exact policy gradient properties in the limit.
\end{remark}

\begin{figure*}[h!]
\centering 
\includegraphics[width=1.0\textwidth]{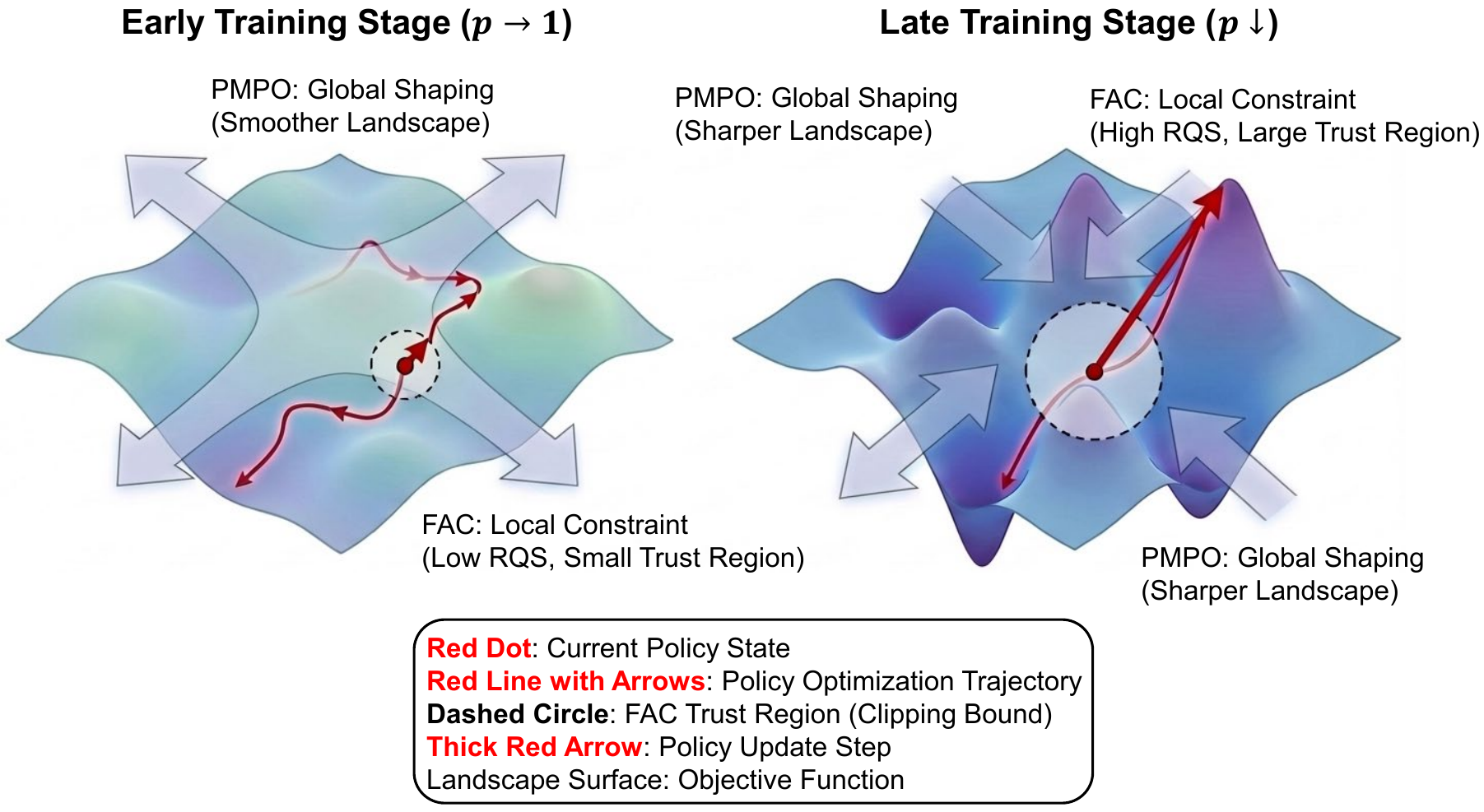}
\caption{Illustration of the synergy of PMPO and FAC.}
\label{fig:synergy}
\end{figure*}

\section{Appendix J: Pseudocode for APMPO} \label{appendix:pseudo}

Based on the description in Section \ref{sec:method}, 
the pseudocode for APMPO is presented in Algorithm
\ref{alg:apmpo}, which outlines its key steps and facilitates the reproducibility of our method.

\begin{algorithm*}[h!]
\small
\caption{Adaptive Power-Mean Policy Optimization (APMPO)}
\label{alg:apmpo}
\begin{flushleft}
\textbf{Input}: Policy model $\pi_\theta$, old policy model $\pi_{\theta_{\text{old}}}$, reference model $\pi_{\theta_{\text{ref}}}$, training data $\mathcal{D}$ comprising queries $q$;

\textbf{Required}: Learning rate $\eta$, KL divergence regularizer $\beta$, group size per query $G$, total training steps $T$, PMPO sensitivity parameter $\gamma$, FAC clipping bounds $(\epsilon_{\text{min}}, \epsilon_{\text{max}})$, fixed lower bound $\epsilon_{\text{low}}$, numerical stability constant $\delta$;

\textbf{Output}: Optimized policy model $\pi_\theta^*$;
\end{flushleft}
\begin{algorithmic}[1]
\For{$t = 1$ to $T$}
    \State $\pi_{\theta_{\text{old}}} \leftarrow \pi_\theta$; \Comment{Update old policy}
    \State Sample a query $q \sim \mathcal{D}$; 
    \State Initialize lists for responses $\mathcal{O} \leftarrow []$, rewards $\mathcal{R} \leftarrow []$, advantages $\mathcal{A} \leftarrow []$;
    \State Sample $G$ responses $\{o_i\}_{i=1}^G \sim \pi_{\theta_{\text{old}}}(\cdot|q)$;
    \State Obtain verifiable rewards $\{R_i\}_{i=1}^G$ for each response;
    \State $\mu_R \leftarrow \frac{1}{G} \sum_{i=1}^{G} R_i$; \Comment{Group average reward for PMPO}
    \State $\sigma_R \leftarrow \sqrt{\frac{1}{G} \sum_{i=1}^{G} (R_i - \mu_R)^2}$; \Comment{Group reward standard deviation for PMPO}
    
    \For{$i = 1$ to $G$}
         \State $\hat{A}_i \leftarrow (R_i - \mu_R) / \sigma_R$; \Comment{Compute group-normalized advantage}
         \State Append $\hat{A}_i$ to $\mathcal{A}$;
    \EndFor
    
    \State {/* --- \textit{\textbf{Core of APMPO: Adaptive Mechanisms}} --- */}
    \State \textbf{// PMPO: Determine adaptive exponent p}
    \State $p \leftarrow \exp(-\gamma \cdot \mu_R)$; \Comment{Eq. (\ref{eq:dynamic_p})}
    
    \State \textbf{// FAC: Determine adaptive clipping bound $\epsilon_{\text{ada}}$}
    \State FSS $\leftarrow \mu_{R} / (\sigma_{R} + \delta)$; \Comment{Eq. (\ref{eq:snr})}
    \State $\epsilon_{\text{ada}} \leftarrow \epsilon_{\text{min}} + (\epsilon_{\text{max}} - \epsilon_{\text{min}}) \cdot \text{FSS}$; \Comment{Eq. (\ref{eq:adaptive_epsilon})}

    \State {/* --- \textit{\textbf{Objective Function Construction}} --- */}
    \State Initialize total loss $\mathcal{J}_{\text{APMPO}}(\theta) \leftarrow 0$;
    \For{$i=1$ to $G$} \Comment{Iterate over all samples in the group}
        \State Let $o_i$ be the $i$-th response with length $|o_i|$ and advantage $\hat{A}_i$;
        \State Initialize token-level magnitude sum $S_i \leftarrow 0$;
        \For{$t=1$ to $|o_i|$}
            \State $r_{i,t}(\theta) \leftarrow \frac{\pi_\theta(o_{i,t}|q, o_{i,<t})}{\pi_{\theta_{\text{old}}}(o_{i,t}|q, o_{i,<t})}$; \Comment{Importance sampling ratio}
            \State $\rho_{i,t}(\theta) \leftarrow \max(1 - \epsilon_{\text{low}}, \min(r_{i,t}(\theta), 1 + \epsilon_{\text{ada}}))$; \Comment{Eq. (\ref{eq:effective_ratio})}
            \State $\phi_{i,t}(\theta) \leftarrow |\min(r_{i,t}(\theta) \cdot \hat{A}_i, \rho_{i,t}(\theta) \cdot \hat{A}_i)|$; \Comment{Eq. (\ref{eq:token_magnitude})}
            \State $S_i \leftarrow S_i + (\phi_{i,t}(\theta))^p$;
        \EndFor
        \State $\mathcal{J}_i(\theta) \leftarrow \left(\frac{1}{|o_i|} S_i\right)^{1/p} \cdot \text{sgn}(\hat{A}_i)$; \Comment{Eq. (\ref{eq:per_sequence_objective})}
        \State $\mathcal{J}_{\text{APMPO}}(\theta) \leftarrow \mathcal{J}_{\text{APMPO}}(\theta) + \mathcal{J}_i(\theta)$;
    \EndFor
    
    \State $\mathcal{J}_{\text{total}}(\theta) \leftarrow \frac{1}{G} \mathcal{J}_{\text{APMPO}}(\theta) - \beta D_{\text{KL}}(\pi_\theta || \pi_{\theta_{\text{ref}}})$; \Comment{Final objective to maximize}
    \State $\theta \leftarrow \theta + \eta \nabla_\theta \mathcal{J}_{\text{total}}(\theta)$; \Comment{Update policy parameters via gradient ascent}
\EndFor
\end{algorithmic}
\end{algorithm*}

\section{Appendix K: Further Analysis}

\subsection{Analysis of Training Efficiency} \label{appendix:complexity_analysis}
The empirical results in Figure~\ref{fig:training}(c) demonstrate that APMPO achieves significant performance gains with negligible additional wall-clock time compared to GRPO. This efficiency is justified by a computational complexity analysis of a single training step.

Let $B$ be the batch size and $G$ be the number of rollouts for each query. The total number of sampled responses per batch is $N = B \times G$. Let $L$ denote the average sequence length. The computational cost of one training step for GRPO can be broken down as follows:

\begin{itemize}
    \item \textbf{Rollout:} This involves $N$ forward passes to generate responses. The cost is given by:
    \begin{equation}
        \mathcal{T}_{\text{rollout}} = N \cdot \mathcal{T}_{\text{forward}} \approx O(B \cdot G \cdot L)
    \end{equation}
    where $\mathcal{T}_{\text{forward}}$ is the cost of one forward pass. This phase constitutes the most computationally expensive component of the process.

    \item \textbf{Advantage Calculation:} For GRPO, this requires computing the mean $\mu_R$ and standard deviation $\sigma_R$ of rewards across $N$ samples, followed by normalizing rewards for all $N$ samples. The complexity is given by:
    \begin{equation}
        \mathcal{T}_{\text{adv\_grpo}} = O(B \cdot G) = O(N)
    \end{equation}

    \item \textbf{Policy Update:} This involves a single forward and backward pass using the $N$ sampled responses to estimate the policy gradient. The cost is given by:
    \begin{equation}
        \mathcal{T}_{\text{update}} = \mathcal{T}_{\text{forward}} + \mathcal{T}_{\text{backward}} \approx O(B \cdot G \cdot L)
    \end{equation}
\end{itemize}

The total complexity for GRPO is therefore given as:
\begin{equation}
    \mathcal{T}_{\text{GRPO}} \approx \mathcal{T}_{\text{rollout}} + \mathcal{T}_{\text{adv\_grpo}} + \mathcal{T}_{\text{update}} \approx O(B \cdot G \cdot L)
\end{equation}

Next, the additional computations introduced by APMPO are analyzed:
\begin{itemize}
    \item \textbf{PMPO Overhead:} To compute the adaptive exponent $p$, the computation requires the batch-level average reward $\mu_R$. This statistic is already obtained in GRPO during advantage normalization. The subsequent $\exp(\cdot)$ operation is a single scalar computation, with a complexity of $O(1)$.
    \begin{equation}
        \mathcal{T}_{\text{overhead\_pmpo}} = O(1)
    \end{equation}

    \item \textbf{FAC Overhead:} To compute FSS, the calculation requires the mean $\mu_{R}$ and standard deviation $\sigma_{R}$ within the batch. This involves iterating through $N$ advantage values. The complexity is:
    \begin{equation}
        \mathcal{T}_{\text{overhead\_fac}} = O(N)
    \end{equation}
\end{itemize}

The total additional complexity introduced by APMPO is $\mathcal{T}_{\text{overhead\_apmpo}} = \mathcal{T}_{\text{overhead\_pmpo}} + \mathcal{T}_{\text{overhead\_fac}} \approx O(N)$.

Finally, the total complexity of APMPO is:
\begin{equation}
    \mathcal{T}_{\text{APMPO}} = \mathcal{T}_{\text{GRPO}} + O(N) \approx O(B \cdot G \cdot L) + O(B \cdot G)
\end{equation}

In the context of LLMs, where the sequence length $L$ is typically large, the additional $O(N)$ overhead remains negligible. This confirms the empirical finding that the adaptive mechanisms in APMPO are computationally lightweight.

\subsection{Analyzing the Synergy of PMPO and FAC} \label{appendix:synergy}
The superior performance of APMPO stems from the hierarchical synergy between PMPO and FAC in regulating optimization dynamics.

As illustrated in Figure \ref{fig:synergy}, PMPO functions at the strategic level by shaping the global optimization landscape. By conditioning the power parameter $p$ on the reward mean $\mu_R$, PMPO determines the learning regime, enabling smooth transitions between aggressive signal exploration and conservative policy consolidation.

Conversely, FAC operates at the tactical level, controlling the trust-region constraints. Guided by FSS, FAC modulates the magnitude of policy updates generated by PMPO, ensuring that step sizes are adaptively calibrated according to the stability of the reward signal.

This two‑tiered control mechanism underpins APMPO’s advantage. While PMPO alone can adapt the update direction, it remains sensitive to random fluctuations in the reward signal, which may cause instability. FAC mitigates this issue by imposing tighter constraints when reward signals are unreliable. Conversely, FAC alone promotes stability but lacks the capacity to reshape the gradient landscape. Their integration allows APMPO to achieve learning that is strategically adaptive.

\subsection{Further Analysis of FSS Components} \label{appendix:FSS}

The empirical superiority of the composite FSS formulation over its isolated components (\textit{i.e.}, the numerator $\mu_{R}$ or the denominator $1/\sigma_{R}$) underscores the necessity of a holistic metric for feedback quality.

Specifically, relying solely on $\mu_{R}$  provides a linear learning signal that biases the optimization toward batches with high mean reward, irrespective of their stability. This approach risks overfitting to batches containing high-reward outliers, leading to potentially erroneous policy updates. Conversely, an objective emphasizing stability (\textit{i.e.}, high $1/\sigma_{R}$) tends to be overly conservative, rewarding consistently poor behavior. Over‑penalizing instability can hinder the emergence of complex reasoning patterns that typically arise during early exploratory stages.

By integrating both components via the specific formulation, FSS allows FAC to distinguish between reliable high-quality signals and unreliable fluctuations. This enables the expansion of clipping bounds only when the learning signal is genuinely trustworthy, thereby ensuring a principled balance between learning efficacy and policy stability.

\subsection{Analysis on Clipping Asymmetry} \label{appendix:clipping asymmetry} 

To justify the design choice of FAC, we conducted an ablation study comparing our asymmetric clipping strategy (\textit{i.e.}, Fixed Lower Bound) against a fully symmetric variant, where both the upper and lower bounds are adaptively adjusted by FSS. The expression of the variant is given as:
\begin{equation}
\rho_{i,t}(\theta) = \max\left[1 - \epsilon_{\text{ada}}, \min(r_{i,t}(\theta), 1 + \epsilon_{\text{ada}}) \right]
\label{eq:effective_ratio1}
\end{equation}

As shown in Table~\ref{tab:ablation_asymmetry}, our default FAC consistently outperformed the symmetric adaptive variant. This performance gap stemmed from the distinct roles of the clipping bounds. While an adaptive upper bound allowed the model to aggressively capitalize on high-quality consensus (\textit{i.e.}, high FSS), an adaptive lower bound creates instability. Relaxing the lower bound allowed the policy to reduce the probability of certain tokens, potentially leading to excessive policy shifts. By keeping the lower bound fixed, FAC acted as a safety anchor, enabling aggressive positive reinforcement while preventing destabilizing negative updates. This confirmed that an asymmetric trust region was more suitable for reasoning tasks.

\begin{table*}[t]
\centering
\small
\setlength{\tabcolsep}{1.5mm}
\begin{tabular}{lccccccc}
\toprule
\textbf{Method} & \textbf{Math500} & \textbf{AIME24} & \textbf{AIME25} & \textbf{AMC23} & \textbf{Minerva} & \textbf{Olympiad} & \textbf{Avg.} \\
\midrule
\multicolumn{8}{c}{\textit{Qwen2.5-Math-1.5B-Instruct}} \\
\midrule
FSS (Adaptive Lower) & 77.5 & \textbf{20.0} & 16.7 & 60.0 & 30.1 & 42.0 & 41.1 \\
\textbf{FSS (Fixed Lower)} & \textbf{78.0} & \textbf{20.0} & \textbf{16.7} & \textbf{62.5} & \textbf{30.5} & \textbf{42.4} & \textbf{41.7} \\
\midrule
\multicolumn{8}{c}{\textit{Qwen2.5-3B-Instruct}} \\
\midrule
FSS (Adaptive Lower) & 68.2 & \textbf{10.0} & \textbf{10.0} & 42.5 & 27.6 & 32.9 & 31.8 \\
\textbf{FSS (Fixed Lower)} & \textbf{68.4} & \textbf{10.0} & \textbf{10.0} & \textbf{45.0} & \textbf{27.9} & \textbf{33.2} & \textbf{32.4} \\
\midrule
\multicolumn{8}{c}{\textit{DeepSeek-R1-Distill-Qwen-1.5B}} \\
\midrule
FSS (Adaptive Lower) & 81.4 & 20.0 & 23.3 & 62.5 & 32.4 & 46.3 & 44.4 \\
\textbf{FSS (Fixed Lower)} & \textbf{81.6} & \textbf{23.3} & \textbf{26.7} & \textbf{65.0} & \textbf{32.7} & \textbf{46.6} & \textbf{46.0} \\
\bottomrule
\end{tabular}
\caption{Ablation study on the asymmetry of FAC. We compare our default asymmetric design (Fixed Lower Bound) against a symmetric variant where both bounds are adaptive. The best results are highlighted in bold.}
\label{tab:ablation_asymmetry}
\end{table*}

\subsection{Details of Adaptive Exponent Formulation} \label{appendix:p}

The superior performance of the exponential-decay formulation for the adaptive exponent $p$ underscores the importance of achieving a smooth and asymptotic transition within the learning objective. Specifically, the primary limitation of the linear variant lies in its constant rate of change. The exponent $p$ responds to fluctuations in $\mu_R$ with a linear sensitivity determined by $\gamma$. This approach treats all changes in reward equally, failing to distinguish between minor reward fluctuations and significant performance shifts that warrant different response magnitudes. 

In contrast, the exponential-decay formulation provides two key advantages. First, it ensures a smooth and non-linear transition, where the rate of change gradually decreases as performance improves. Second, $p$ asymptotically approaches but never reaches zero. This property guarantees persistent exploratory pressure, preventing the policy from collapsing into overly deterministic behavior. Owing to its inherent smoothness and non-saturating behavior, the exponential-decay formulation represents a more stable mechanism for adaptive policy optimization.

\subsection{Empirical Analysis of Output Diversity}
\label{sec:diversity_analysis}

To substantiate the claim that APMPO mitigates the entropy collapse observed in GRPO and circumvents the conservatism of GMPO, the output diversity of models on the MATH500 dataset was analyzed. Notably, $8$ solutions were generated for each prompt using three complementary metrics:

\begin{itemize}
    \item \textbf{Self-BLEU:} The BLEU-4 score of each generated solution was calculated against others in the same sampled group. A lower Self-BLEU signifies reduced lexical overlap, suggesting that the model avoids simply repeating memorized templates.
    \item \textbf{Average Cosine Similarity (ACS):} Pairwise ACS was computed between the sentence embeddings\footnote{\url{https://huggingface.co/Qwen/Qwen3-Embedding-4B}} of all generated solutions per prompt. Lower ACS values indicate greater semantic diversity in the reasoning process.
    \item \textbf{GPT-based Diversity Score:} GPT-5 was employed to evaluate the distinctness of valid reasoning paths. For correct solutions, a score from 0 to 3 was assigned based on mathematical distinctness (detailed in the \textit{GPT-based prompt}). A higher score implies successful exploration of fundamentally different solving strategies.
\end{itemize}

\begin{table}[h]
    \centering
    \small
    \setlength{\tabcolsep}{1.4mm}
    \begin{tabular}{l|ccc}
    \toprule
    \textbf{Method} & \textbf{ACS} ($\downarrow$) & \textbf{Self-BLEU} ($\downarrow$) & \textbf{GPT-Diversity} ($\uparrow$) \\
    \midrule
    GRPO            & 0.84                        & 0.76                              & 1.22 \\
    DAPO            & 0.79                        & 0.71                              & 1.35 \\
    GMPO            & 0.75                        & 0.65                              & 1.41 \\
    \textbf{APMPO}  & \textbf{0.66}               & \textbf{0.54}                     & \textbf{1.62} \\
    \bottomrule
    \end{tabular}
    \caption{Diversity metrics on the MATH-500 dataset using Qwen2.5-Math-1.5B-Instruct. Comparisons are made against state-of-the-art RLVR-based baselines. $\downarrow$ indicates lower is better, and $\uparrow$ indicates higher is better.}
    \label{tab:diversity_metrics}
\end{table}

\paragraph{Analysis.} As presented in Table \ref{tab:diversity_metrics}, GRPO exhibited the highest ACS and Self-BLEU, confirming the ``mode collapse'' phenomenon indicated by the entropy curves in Figure \ref{fig:training}(b). The arithmetic mean objective in GRPO disproportionately reinforced the first successful path, thereby inducing early policy convergence to a single solution pattern. While DAPO and GMPO improved diversity, it remained more conservative than APMPO. Crucially, APMPO achieved superior performance across all diversity metrics, further validating the efficacy of PMPO and FAC.

\paragraph{GPT-based Prompt.} To ensure rigorous evaluation, the following prompt was employed for the GPT-based Diversity Score:

\textit{``Given the following set of correct solutions to a math problem, evaluate the diversity of the reasoning methods used. Assign a single integer score from 0 to 3 based on the following criteria:
\begin{itemize}
    \item \textbf{0:} Solutions use identical logic and phrasing (High Repetition).
    \item \textbf{1:} Solutions use the same underlying logic but different phrasing.
    \item \textbf{2:} Solutions use slightly different logical steps or intermediate derivations.
    \item \textbf{3:} Solutions employ fundamentally different mathematical approaches (e.g., Coordinate Geometry vs. Synthetic Geometry, or Induction vs. Direct Proof).
\end{itemize}
Output only the single numerical score.''}

\section{Appendix J: More Experimental Results}

\subsection{Results on SQL Generation and Multi-modal Reasoning} \label{appendix:geoSQL}
The experimental results on SQL generation (Spider and BIRD) and multi-modal reasoning (Geometry3K) are shown in Table \ref{tab:generalization_results_apmpo}.

\begin{table*}[t]
\centering
\begin{tabular}{lccc}
\toprule
\textbf{Method} & \textbf{Geometry3K} & \textbf{Spider} & \textbf{BIRD} \\
\midrule
Base     & 25.7              & 70.2              & 43.6      \\
GRPO     & 35.6$_{\pm 0.8}$  & 73.8$_{\pm 0.5}$  & 55.7$_{\pm 0.9}$  \\
DAPO     & 36.5$_{\pm 0.7}$  & 75.6$_{\pm 0.6}$  & 58.2$_{\pm 0.8}$ \\
GMPO     & 37.2$_{\pm 0.6}$  & 74.8$_{\pm 0.5}$  & 57.8$_{\pm 0.7}$  \\
\textbf{APMPO} & \textbf{37.9}$_{\pm 0.5}$ & \textbf{76.4}$_{\pm 0.4}$ & \textbf{60.6}$_{\pm 0.6}$ \\
\bottomrule
\end{tabular}
\caption{Generalization performance on multi-modal reasoning (Geometry3K) and SQL generation (Spider, BIRD). The results are given as mean and standard deviation across 3 random seeds.}
\label{tab:generalization_results_apmpo}
\end{table*}

\subsection{Effectiveness of PMPO and FAC} \label{appen:component}

The experimental results regarding the efficacy of different components in APMPO are presented in Table \ref{tab:ablation1}.

\begin{table*}[t]
\centering
\small
\setlength{\tabcolsep}{2.8mm}
\begin{tabular}{cccccccc}
\toprule
\textbf{Method} & \textbf{Math500} & \textbf{AIME24} & \textbf{AIME25} & \textbf{AMC23} & \textbf{Minerva} & \textbf{Olympiad} & \textbf{Avg.}  \\
\midrule
\multicolumn{8}{c}{\textit{Qwen2.5-Math-1.5B-Instruct}} \\
\midrule
GRPO & 75.2 & 13.3 & 13.3 & 52.5 & 29.4 & 39.0 & 37.5 \\
+ FAC & 76.4 & 16.7 & 13.3 & 57.5 & 29.8 & 40.1 & 39.0 \\
+ PMPO & 77.2 & \textbf{20.0} & \textbf{16.7} & 60.0 & 30.1 & 41.4 & 40.9 \\
APMPO & \textbf{78.0} & \textbf{20.0} & \textbf{16.7} & \textbf{62.5} & \textbf{30.5} & \textbf{42.4} & \textbf{41.7} \\
\midrule
\multicolumn{8}{c}{\textit{Qwen2.5-3B-Instruct}} \\
\midrule
GRPO & 66.0 & 6.7 & 6.7 & 40.0 & 25.4 & 31.5 & 29.4 \\
+ FAC & 67.2 & 6.7 & \textbf{10.0} & 42.5 & 26.5 & 32.0 & 30.8 \\
+ PMPO & 67.8 & \textbf{10.0} & \textbf{10.0} & \textbf{45.0} & 27.2 & 32.6 & 32.1 \\
APMPO & \textbf{68.4} & \textbf{10.0} & \textbf{10.0} & \textbf{45.0} & \textbf{27.9} & \textbf{33.2} & \textbf{32.4} \\
\midrule
\multicolumn{8}{c}{\textit{DeepSeek-R1-Distill-Qwen-1.5B}} \\
\midrule
GRPO & 75.4 & 13.3 & 20.0 & 57.5 & 29.8 & 43.2 & 39.9 \\
+ FAC & 77.6 & 16.7 & 20.0 & 60.0 & 30.9 & 44.1 & 41.6 \\
+ PMPO & 79.4 & 20.0 & 23.3 & 62.5 & 31.6 & 45.1 & 43.7 \\
APMPO & \textbf{81.6} & \textbf{23.3} & \textbf{26.7} & \textbf{65.0} & \textbf{32.7} & \textbf{46.6} & \textbf{46.0} \\
\bottomrule
\end{tabular}
\caption{Ablation study on multiple mathematical reasoning benchmarks using different modules. The best results are highlighted in bold, and we reported Pass@1 score.}
\label{tab:ablation1}
\end{table*}

\subsection{Ablation of FSS Components} \label{appen:FSS}
The experimental results regarding different FSS components are shown in Table \ref{tab:ablation2}.

\begin{table*}[t]
\centering
\small
\setlength{\tabcolsep}{2.8mm}
\begin{tabular}{cccccccc}
\toprule
\textbf{Method} & \textbf{Math500} & \textbf{AIME24} & \textbf{AIME25} & \textbf{AMC23} & \textbf{Minerva} & \textbf{Olympiad} & \textbf{Avg.}  \\
\midrule
\multicolumn{8}{c}{\textit{Qwen2.5-Math-1.5B-Instruct}} \\
\midrule
Only $\sigma_{R}$ & 77.2 & 16.7 & 16.7 & 60.0 & 29.8 & 41.4 & 40.3 \\
Only $\mu_{R}$ & 77.6 & \textbf{20.0} & \textbf{16.7} & 60.0 & 30.1 & 42.0 & 41.1 \\
FSS & \textbf{78.0} & \textbf{20.0} & \textbf{16.7} & \textbf{62.5} & \textbf{30.5} & \textbf{42.4} & \textbf{41.7} \\
\midrule
\multicolumn{8}{c}{\textit{Qwen2.5-3B-Instruct}} \\
\midrule
Only $\sigma_{R}$ & 67.6 & 6.7 & \textbf{10.0} & 42.5 & 26.8 & 32.6 & 31.0 \\
Only $\mu_{R}$ & 67.8 & \textbf{10.0} & \textbf{10.0} & 42.5 & 27.2 & 32.8 & 31.7 \\
FSS & \textbf{68.4} & \textbf{10.0} & \textbf{10.0} & \textbf{45.0} & \textbf{27.9} & \textbf{33.2} & \textbf{32.4} \\
\midrule
\multicolumn{8}{c}{\textit{DeepSeek-R1-Distill-Qwen-1.5B}} \\
\midrule
Only $\sigma_{R}$ & 80.8 & 20.0 & 20.0 & 62.5 & 32.0 & 45.8 & 43.5 \\
Only $\mu_{R}$ & 81.2 & 20.0 & 23.3 & 62.5 & 32.4 & 46.1 & 44.3 \\
FSS & \textbf{81.6} & \textbf{23.3} & \textbf{26.7} & \textbf{65.0} & \textbf{32.7} & \textbf{46.6} & \textbf{46.0} \\
\bottomrule
\end{tabular}
\caption{Ablation study on multiple mathematical reasoning benchmarks using different variants of FSS. The best results are highlighted in bold, and we reported Pass@1 score.}
\label{tab:ablation2}
\end{table*}

\subsection{Ablation of Adaptive Exponent Formulation} \label{appendix:exponent}
The experimental results regarding the adaptive exponent formulation in PMPO are presented in Table \ref{tab:ablation3}.

\begin{table*}[t]
\centering
\small
\setlength{\tabcolsep}{2.8mm}
\begin{tabular}{cccccccc}
\toprule
\textbf{Method} & \textbf{Math500} & \textbf{AIME24} & \textbf{AIME25} & \textbf{AMC23} & \textbf{Minerva} & \textbf{Olympiad} & \textbf{Avg.}  \\
\midrule
\multicolumn{8}{c}{\textit{Qwen2.5-Math-1.5B-Instruct}} \\
\midrule
Linear $p$ & 77.2 & 16.7 & \textbf{16.7} & 57.5 & 29.8 & 41.8 & 40.0 \\
Ours & \textbf{78.0} & \textbf{20.0} & \textbf{16.7} & \textbf{62.5} & \textbf{30.5} & \textbf{42.4} & \textbf{41.7} \\
\midrule
\multicolumn{8}{c}{\textit{Qwen2.5-3B-Instruct}} \\
\midrule
Linear $p$ & 68.0 & 6.7 & \textbf{10.0} & 40.0 & 27.2 & 32.8 & 30.8 \\
Ours & \textbf{68.4} & \textbf{10.0} & \textbf{10.0} & \textbf{45.0} & \textbf{27.9} & \textbf{33.2} & \textbf{32.4} \\
\midrule
\multicolumn{8}{c}{\textit{DeepSeek-R1-Distill-Qwen-1.5B}} \\
\midrule
Linear $p$ & 81.2 & \textbf{23.3} & 23.3 & 60.0 & 32.0 & 46.1 & 44.3 \\
Ours & \textbf{81.6} & \textbf{23.3} & \textbf{26.7} & \textbf{65.0} & \textbf{32.7} & \textbf{46.6} & \textbf{46.0} \\
\bottomrule
\end{tabular}
\caption{Ablation study on multiple mathematical reasoning benchmarks using different $p$ formulations. The best results are highlighted in bold, and we reported Pass@1 score.}
\label{tab:ablation3}
\end{table*}

\subsection{Ablation of $\gamma$ in PMPO}
\label{appen:gamma}
The experimental results regarding the parameter $\gamma$ in PMPO are shown in Table \ref{tab:ablation4}. Note that $(\epsilon_\text{min}, \epsilon_\text{max})=(0.2,0.4)$ was kept when comparing different values of $\gamma$.

\begin{table*}[t]
\centering
\small
\setlength{\tabcolsep}{2.8mm}
\begin{tabular}{cccccccc}
\toprule
\textbf{$\gamma$} & \textbf{Math500} & \textbf{AIME24} & \textbf{AIME25} & \textbf{AMC23} & \textbf{Minerva} & \textbf{Olympiad}& \textbf{Avg.}  \\
\midrule
\multicolumn{8}{c}{\textit{Qwen2.5-Math-1.5B-Instruct}} \\
\midrule
0.2 & 76.8 & 13.3 & 13.3 & 55.0 & 28.6 & 41.4 & 38.1 \\
0.4 & 77.2 & 13.3 & 13.3 & 57.5 & 29.4 & 41.8 & 38.8 \\
0.6 & 77.6 & 16.7 & \textbf{16.7} & 60.0 & 30.1 & 42.1 & 40.5 \\
0.8 & \textbf{78.0} & \textbf{20.0} & \textbf{16.7} & \textbf{62.5} & \textbf{30.5} & \textbf{42.4} & \textbf{41.7} \\
1.0 & 77.8 & \textbf{20.0} & 13.3 & \textbf{62.5} & 29.8 & 41.5 & 40.8 \\
\midrule
\multicolumn{8}{c}{\textit{Qwen2.5-3B-Instruct}} \\
\midrule
0.2 & 67.3 & 3.3 & 6.7 & 35.0 & 25.7 & 32.5 & 28.4 \\
0.4 & 67.6 & 6.7 & 6.7 & 40.0 & 26.5 & 32.8 & 30.1 \\
0.6 & 68.2 & 6.7 & \textbf{10.0} & 42.5 & 27.2 & 32.9 & 31.3 \\
0.8 & \textbf{68.4} & \textbf{10.0} & \textbf{10.0} & \textbf{45.0} & \textbf{27.9} & \textbf{33.2} & \textbf{32.4} \\
1.0 & 68.0 & \textbf{10.0} & 6.7 & 45.0 & 27.6 & \textbf{33.2} & 31.8 \\
\midrule
\multicolumn{8}{c}{\textit{DeepSeek-R1-Distill-Qwen-1.5B}} \\
\midrule
0.2 & 79.6 & 20.0 & 20.0 & 60.0 & 31.3 & 45.1 & 42.7 \\
0.4 & 80.3 & 20.0 & 23.3 & 62.5 & 31.6 & 45.6 & 43.9 \\
0.6 & 81.1 & \textbf{23.3} & 23.3 & \textbf{65.0} & 32.4 & 46.1 & 45.2 \\
0.8 & \textbf{81.6} & \textbf{23.3} & \textbf{26.7} & \textbf{65.0} & \textbf{32.7} & \textbf{46.6} & \textbf{46.0} \\
1.0 & 80.8 & 20.0 & 26.7 & 62.5 & 32.0 & 45.8 & 44.6 \\
\bottomrule
\end{tabular}
\caption{Ablation study on multiple mathematical reasoning benchmarks using different values of $\gamma$. The best results are highlighted in bold, and we reported Pass@1 score.}
\label{tab:ablation4}
\end{table*}

\subsection{Ablation of $(\epsilon_\text{min}, \epsilon_\text{max})$ in FAC} \label{appen:epsilon}
The experimental results regarding the parameters $(\epsilon_\text{min}, \epsilon_\text{max})$ in FAC are shown in Table \ref{tab:ablation5}. Note that $\gamma=0.8$ was kept when comparing different values of $(\epsilon_\text{min}, \epsilon_\text{max})$.

\begin{table*}[t]
\centering
\small
\setlength{\tabcolsep}{2.8mm}
\begin{tabular}{cccccccc}
\toprule
\textbf{$(\epsilon_\text{min}, \epsilon_\text{max})$} & \textbf{Math500} & \textbf{AIME24} & \textbf{AIME25} & \textbf{AMC23} & \textbf{Minerva} & \textbf{Olympiad}& \textbf{Avg.}  \\
\midrule
\multicolumn{8}{c}{\textit{Qwen2.5-Math-1.5B-Instruct}} \\
\midrule
$(0.1,0.3)$ & 77.2 & 16.7 & 13.3 & 55.0 & 29.4 & 41.2 & 38.8 \\
$(0.1,0.4)$ & 77.8 & \textbf{20.0} & 13.3 & 60.0 & 30.1 & 42.1 & 40.6 \\
$(0.2,0.3)$ & 77.6 & 16.7 & 13.3 & 60.0 & 29.8 & 41.4 & 39.8 \\
$(0.2,0.4)$ & \textbf{78.0} & \textbf{20.0} & \textbf{16.7} & \textbf{62.5} & \textbf{30.5} & \textbf{42.4} & \textbf{41.7} \\
\midrule
\multicolumn{8}{c}{\textit{Qwen2.5-3B-Instruct}} \\
\midrule
$(0.1,0.3)$ & 67.6 & 6.7 & 6.7 & 40.0 & 26.5 & 32.8 & 30.1 \\
$(0.1,0.4)$ & 68.0 & \textbf{10.0} & \textbf{10.0} & 42.5 & 27.2 & 32.9 & 31.8 \\
$(0.2,0.3)$ & 67.8 & 6.7 & \textbf{10.0} & 40.0 & 27.2 & 32.9 & 30.8 \\
$(0.2,0.4)$ & \textbf{68.4} & \textbf{10.0} & \textbf{10.0} & \textbf{45.0} & \textbf{27.9} & \textbf{33.2} & \textbf{32.4} \\
\midrule
\multicolumn{8}{c}{\textit{DeepSeek-R1-Distill-Qwen-1.5B}} \\
\midrule
$(0.1,0.3)$ & 80.7 & 20.0 & 23.3 & 60.0 & 31.6 & 45.8 & 43.5 \\
$(0.1,0.4)$ & 81.2 & 20.0 & \textbf{26.7} & \textbf{65.0} & 32.4 & 46.3 & 45.3 \\
$(0.2,0.3)$ & 80.4 & \textbf{23.3} & 23.3 & 62.5 & 32.0 & 46.1 & 44.6 \\
$(0.2,0.4)$ & \textbf{81.6} & \textbf{23.3} & \textbf{26.7} & \textbf{65.0} & \textbf{32.7} & \textbf{46.6} & \textbf{46.0} \\
\bottomrule
\end{tabular}
\caption{Ablation study on multiple mathematical reasoning benchmarks using different values of $(\epsilon_\text{min}, \epsilon_\text{max})$. The best results are highlighted in bold, and we reported Pass@1 score.}
\label{tab:ablation5}
\end{table*}

\end{document}